\newcommand{\longversion}[1]{#1}
\newcommand{\shortversion}[1]{}

\documentclass{new_tlp}
\usepackage{mathptmx}

\DeclareMathAlphabet{\mathcal}{OMS}{cmsy}{m}{n}

\usepackage{amssymb,amsfonts,amsmath}
\usepackage{microtype}
\usepackage{caption}
\usepackage{mathtools}
\usepackage{xspace}
\usepackage{algorithm}
\usepackage{algorithmic}
\usepackage{booktabs}
\usepackage{adjustbox}
\usepackage{subcaption}
\usepackage[hyphens,spaces]{url}

\newcommand{\assume}[2]{\ensuremath{#1}[#2]}
\newcommand{\sddcompl}{\Phi_{\lp}}
\newcommand{\sddcomplL}{\Phi_{\assume{\lp}{L}}}
\newcommand{\CG}{\mathcal{G}_{\assume{\lp}{L}}}

\newcommand{\defi}{\coloneqq}                      %
\newcommand{\lp}{\Pi}                              %
\newcommand{\A}[1][\lp]{\mathcal{\mathcal{L}}(#1)}           %
\newcommand{\AS}[1][\lp]{\mathcal{AS}(#1)}         %
\DeclareMathOperator{\Mod}{\mathcal{M}}         %
\DeclareMathOperator{\SMod}{\mathcal{S}}         %
\DeclareMathOperator{\ic}{ic}         %

\def\hy{\hbox{-}\nobreak\hskip0pt}
\newcommand{\SP}[1][\lp]{\mathcal{S}(#1)}          %
\newcommand{\ass}{L}                          %
\newcommand{\pneg}{\neg} %
\newcommand{\sd}[1][\lp]{\Phi_{#1}}                %
\newcommand{\cg}[1][\lp]{\ensuremath{\mathcal{G}(#1)}}         %

\DeclareMathOperator{\compl}{comp}
\newcommand{\val}{\ensuremath{\mathit{val}}}

\newcommand{\Card}[1]{\ensuremath{|#1|}}

\let\phi=\varphi
\let\epsilon=\varepsilon

\newcommand{\atom}{a}%
\newcommand{\at}{at}%
\newcommand{\vars}{\text{vars}} %

\newcommand{\complexityClassFont}[1]{\ensuremath{\mathrm{#1}}}
\newcommand{\NP}{\text{\complexityClassFont{NP}}\xspace}

\newcommand{\PH}{\text{\complexityClassFont{PH}}\xspace}

\newcommand{\Ptime}{\complexityClassFont{P}\xspace}

\newcommand{\numberDotP}{\complexityClassFont{\#{\cdot}\Ptime}\xspace}

\newcommand{\PS}{\ensuremath{\Ptime^{\numberDotP}}\xspace}

{%
  \addtocounter{theorem}{-1}
  \endgroup
}%

{%
  \addtocounter{theorem}{-1}
  \endgroup
}%

{%
  \addtocounter{theorem}{-1}
  \endgroup
}%

\newtheorem{definition}{Definition}
\newtheorem{theorem}{Theorem}
\newtheorem{observation}{Observation}
\newtheorem{corollary}{Corollary}
\newtheorem{example}{Example}
\AtBeginEnvironment{example}{%
}
\AtEndEnvironment{example}{\hfill$\triangle$\endexample}

\usepackage{tikz}

\newcommand\bcmdtab{\noindent\bgroup\tabcolsep=0pt%
  \begin{tabular}{@{}p{10pc}@{}p{20pc}@{}}}
\newcommand\ecmdtab{\end{tabular}\egroup}

  \title[IASCAR: Incremental Answer Set Counting by Anytime Refinement]%
        {IASCAR: Incremental Answer Set Counting by Anytime Refinement}

  \author[J. K. Fichte et al.]
         {%
	JOHANNES K. FICHTE\\
        Link\"oping University, Department of Computer Science (IDA),
        Sweden\\
          	\email{johannes.fichte@liu.se}
          	\and SARAH ALICE GAGGL\\
          	TU Dresden, Logic Programming and Argumentation Group, Germany\\
          	\email{sarah.gaggl@tu-dresden.de}
          	\and MARKUS HECHER\\
          	Massachusetts Institute of Technology, USA\\
          	\email{hecher@mit.edu}
          	\and DOMINIK RUSOVAC\\
          	TU Dresden, Logic Programming and Argumentation Group, Germany\\
          	\email{dominik.rusovac@tu-dresden.de}
		 }

\jdate{March 2003}
\pubyear{2003}
\pagerange{\pageref{firstpage}--\pageref{lastpage}}
\doi{S1471068401001193}

\newtheorem{lemma}{Lemma}[section]

\newcommand{\myhighlight}[1]{#1}

\usepackage{cleveref}
\crefformat{footnote}{#2\footnotemark[#1]#3}

\begin{document}

\label{firstpage}

\maketitle

  \begin{abstract}
	Answer set programming (ASP) is a popular declarative programming paradigm
	with various applications.
        Programs can easily have many answer sets that cannot be
        enumerated in practice, but counting still allows quantifying
        solution spaces.
        If one counts under assumptions on literals, one
	obtains a tool to comprehend parts of the solution space, so-called
	\emph{answer set navigation}. However, navigating through parts of the solution
	space requires counting many times, which is expensive in theory. %
	\emph{Knowledge compilation} compiles instances into representations on
	which counting works in polynomial time. However, these techniques exist
	only for CNF formulas, and compiling ASP programs into CNF formulas can
	introduce an exponential overhead.
        This paper introduces a technique to iteratively count answer
        sets under assumptions on knowledge compilations of CNFs that
        encode supported models.
        Our anytime technique uses the inclusion-exclusion principle
        to improve bounds by over- and undercounting systematically.
        In a preliminary empirical analysis, we demonstrate promising
        results.  After compiling the input (offline phase), our
        approach quickly (re)counts.
  \end{abstract}

  \begin{keywords}
	  ASP, answer set counting, knowledge compilation
  \end{keywords}

\section{Introduction}
\emph{Answer set programming
(ASP)}\shortversion{~\cite{MarekTruszczynski99}}\longversion{~\cite{MarekTruszczynski99,Niemela99,BrewkaEiterTruszczynski11}}
is a widely used declarative problem modeling and solving paradigm with many
applications in artificial intelligence such as knowledge representation, planning, and
many more\longversion{~\cite{Baral03,PontelliSonBaralGelfond12}}. It is widely
used to solve difficult search problems while allowing compact
modeling~\cite{GebserKaufmannSchaub12a}.
In ASP, a problem is represented as a set of rules, called \emph{logic
  program}, over atoms. Models of a program under the stable
semantics\longversion{~\cite{GelfondLifschitz88,GelfondLifschitz91}} form its
solutions, so-called \emph{answer sets}.
Beyond the search for one solution or an optimal solution, an
increasingly popular
question is counting answer sets, which provides extensive
applications for quantitative reasoning.
For example, counting is crucial for probabilistic logic
programming,~c.f.,~\shortversion{\cite{FierensEtAl15,WangLee15}}%
\longversion{\cite{FierensEtAl15,WangLee15,LeeWang2015}} %
or encoding Bayesian networks and their
inference~\cite{SangBeameKautz05a}.

Interestingly, counting also facilitates more fine-grained reasoning
modes between brave and cautious reasoning. To this end, one examines
the ratio of an atom occurring in answer sets over all answer sets,
which yields a notion of \emph{plausibility} of an
atom. %
When considering sets of literals, which represent assumptions,
one obtains a detailed tool to \emph{comprehend search spaces} that
contain a large number of answer
sets~\cite{FichteGagglRusovac21}\shortversion{.}\longversion{,~e.g., for configuration
problems~\cite{dimopoulos1997encoding,lifschitz1999action,nogueira2001prolog}.}
However, already for ground normal programs, answer set counting is
$\numberDotP$-complete\longversion{~\cite{FichteEtAl17}}, making it
harder than decision problems.
Recall that brave reasoning is just $\NP$-complete, but by Toda's
Theorem we know that $\PH \subseteq \PS$\longversion{~\cite{Toda91}}
where $\bigcup_{k \in \mathbb{N}}\Delta _{k}^P = \PH$ and
$\NP \subseteq \Delta^P_2 =
\Ptime^\NP$\longversion{~\cite{Stockmeyer76}}.
Approximate counting is in fact
easier,~i.e.,~$\text{approx}\hy\numberDotP \subseteq \text{BPP}^\NP
\subseteq \Sigma^P_3$\longversion{~\cite{Lautemann1983,Sipser1983,Stockmeyer1983}},
and approximate answer set counters have very recently been suggested~\cite{KabirEverardoShukla22}.
Still, when navigating large search spaces, we need to
count %
answer sets many times rendering such tools
conceptually ineffective.
There, knowledge compilation comes in handy%
\longversion{~\cite{Darwiche04a}}\shortversion{~\cite{darwiche2002knowledge}}.

In \emph{knowledge compilation}, computation is split in two
phases. Formulas are compiled in a potentially very expensive step
into a representation in an \emph{offline phase} and reasoning
is carried out in polynomial time on such representations in
an \emph{online phase}. Such a conceptual framework would be perfectly
suited when answer sets are counted many times, 
providing us with quick re-counting.
While we can translate programs into propositional
formulas\longversion{~\cite{LeeLifschitz03,Lee05,JanhunenNiemela11}} and directly
apply techniques from propositional formulas\longversion{~\cite{LagniezMarquis17a}},
it is widely known that one can easily run into an exponential
blowup~\cite{LifschitzRazborov06} or introduce level
mappings\longversion{~\cite{Janhunen06}} that are oftentimes large grids and hence
expensive for counters.
In practice, solvers that find one answer set or optimal answer sets
can %
avoid a blowup by computing supported models, which can
be encoded into propositional formulas with limited overhead, and
implementing propagators on top~\cite{GebserKaufmannSchaub09}.

In this paper, we explore a counterpart of a propagator-style
approach %
for counting answer sets. %
We encode finding supported models as a propositional formula and use a
knowledge compiler to obtain, in an offline phase, a representation,
which allows us to construct a counting graph that in turn can be used
to  compute the number of supported models efficiently. The resulting counting graph can be large but evaluated in parallel.
Counting supported models only provides an upper bound on the number of answer sets.
\longversion{Therefore, we}\shortversion{We} suggest a combinatorial technique
to systematically improve bounds by over- and undercounting while incorporating
the external support, whose absence can be seen as the cause of overcounting in the
first place. Our technique can be used to approximate the counts but also
provides the exact count on the number of answer sets when taking the entire
external support into account.
\paragraph{Contributions.} Our main contributions are as follows.
\begin{enumerate}
\item We consider knowledge compilation from an ASP perspective. We
  recap features such as counting under assumptions, known as conditioning,
  that make knowledge compilations (sd-DNNFs) quite suitable for
  navigating search spaces.
  We suggest a domain-specific technique to compress counting graphs
  that were constructed for supported models using Clark's completion.
\item We establish a novel combinatorial algorithm that takes an
  sd-DNNF of a completion formula and allows for systematically
  improving bounds by over- and undercounting. The technique
  identifies not supported atoms and compensates for overcounting on the
  sd-DNNF.
\item We apply our approach to instances tailored to navigate
  incomprehensible answer set search spaces. 
  While the problem is challenging in general, we %
  demonstrate feasibility and promising results on quickly (re-)counting.
\end{enumerate}

\paragraph{Related Works.}
Previous work~\cite{BogaertsBroeck2015} considered knowledge
compilation for logic programs. There an eager incremental
approximation technique incrementally computes the result whereas our
approach can be seen as an incremental lazy approach on the counting
graph. Moreover, the technique by Bogarts and Broeck focuses on
well-founded models and stratified negation, which does not work for
normal programs in general without translating ASP programs into CNFs
directly.
Note that common reasoning problems on answer set programs without
negation can be solved in polynomial time\longversion{~\cite{Truszczynski11}}.
Model counting can significantly benefit from preprocessing
techniques\longversion{~\cite{LagniezLoncaMarquis16a,LagniezMarquis14a}}, which
eliminate variables.
Widely used propositional knowledge compilers are
c2d\shortversion{~\cite{darwiche1999compiling}}\longversion{~\cite{Darwiche04a}} and d4\shortversion{~\cite{LagniezMarquis17a}}.
 \myhighlight{%
  Very recent works consider enumerating answer
  sets~\cite{AlvianoDodaroFiorentino23a}, which can be beneficial for
  counting if the number of answer sets is sufficiently low.
  More advanced enumeration techniques have also recently been studied
  for propositional
  satisfiability~\cite{MasinaSpallittaSebastiani23a,SpallittaSebastianiBiere23a}.
}

\paragraph{Prior Work.} %
This paper extends the conference publication~\cite{FichteGHR22}.
  The paper contains more elaborate examples and proofs that have been
  omitted in the preliminary version.
  We now provide an empirical evaluation on relevant instances and
  instances that have been used for counting in previous works.  We
  formulate detailed questions and hypotheses for our algorithm's
  implementation and evaluation. Now, our evaluation incorporates two
  instance sets containing a large number of instances, and we compare
  our approach to state-of-the-art model counters.

\section{Preliminaries}
We assume familiarity with propositional
satisfiability~\cite{Kleine-BuningLettmann99}, graph
theory~\cite{BondyMurty08}, and propositional
ASP~\cite{GebserKaufmannSchaub12a}.  Recall that a \emph{cycle}~$C$ on
a (di)graph~$G$ is a (directed) walk of~$G$ where the first and the
last vertex coincide.  For cycle~$C$, we let $V_C$ be its vertices and
$\mathit{cycles}(G) \defi \{ V_C \mid C \text{ is a cycle of } G\}$.
We consider %
propositional \emph{variables} and mean by formula a propositional
formula.  By $\top$ and $\bot$ we refer to the variables that are
always evaluated to $1$ or $0$ (constants). %
A literal is an atom $\atom$ or its negation~$\neg a$, and
$\vars(\phi)$ denotes the set of variables that occur in
formula~$\phi$.  The set of models of a formula~$\varphi$ is given by
$\Mod(\varphi)$. %
\longversion{%
  Below, we introduce the necessary background and notation
  used in the paper for ASP, and knowledge
  compilation.
}

\newcommand{\reduct}[2]{\ensuremath{#1}_{#2}}

\longversion{\paragraph{Answer Set Programming.}  Let us recall basic
  notions of ASP, for further details we refer to standard
  texts~\cite{GebserKaufmannSchaub12a}. 
} 
\shortversion{%
\paragraph{Answer Set Programming (ASP).}
} %
In the context of ASP, we usually say atom instead of variable.
A \emph{(propositional logic) program} $\lp$ is a finite set of
\emph{rules} $r$ of the form
$$\atom_0 \leftarrow \atom_{1}, \ldots ,\atom_m, \pneg
\atom_{m+1}, \ldots, \pneg \atom_n$$
where $0 \leq m \leq n$ and $\atom_0, \ldots, \atom_n$ are atoms and
usually omit~$\top$ and~$\bot$.
For a rule $r$, we define $H(r) \defi \{\atom_0\}$ called \emph{head}
of~$r$. The \emph{body} %
consists of $B^+(r) \defi \{\atom_{1}, \dots, \atom_m\}$ and
$B^-(r) \defi \{ \atom_{m+1}, \dots,
\atom_n\}$. %
The set~$\at(r)$ of atoms of~$r$ consists
of~$H(r)\cup B^+(r) \cup B^-(r)$. %
Let $\lp$ be a program. Then, we let the set~$\at(\lp) \defi
\bigcup_{r \in \lp}\at(r)$ of~$\lp$
contain its atoms.
Its \emph{positive dependency
  digraph}~$\mathit{DP}(\lp)=(V,E)$ is defined by $V \defi \at(\lp)$ and $E
\defi \{(a_1,a_0) \mid \atom_1 \in B^+(r), \atom_0 \in H(r), r \in
\lp\}$. 
The \emph{cycles of~$\Pi$} are given by~$\mathit{cycles}(\Pi)\defi\mathit{cycles}(\mathit{DP}(\Pi))$.
$\lp$ is \emph{tight}, if $\mathit{DP}(\lp)$ is acyclic.
An \emph{interpretation} of $\lp$ is a set~$I\subseteq \at(\lp)$ of atoms.
$I$ \emph{satisfies} a rule $r \in \lp$ if
$H(r) \cap I \neq \emptyset$ whenever $B^+(r) \subseteq I$ and
$B^-(r) \cap I = \emptyset$. $I$ satisfies~$\lp$, if $I$ satisfies
each rule~$r \in \lp$.
The \emph{GL-reduct}~$\reduct{\lp}{I}$ is defined by %
$\reduct{\lp}{I} \defi \{H(r) \leftarrow B^+(r) \mid I \cap B^-(r) =
\emptyset, r \in \lp\}$.
$I$ is an \emph{answer set}, sometimes also called stable model, if
$I$ satisfies~$\reduct{\lp}{I}$ and~$I$ is subset-minimal.
The \emph{completion}\longversion{~\cite{clark1978negation}}
of~$\Pi$ is the propositional formula
$$\compl(\Pi)\defi \bigwedge_{a \in \at(\lp)} \atom
\leftrightarrow \bigvee_{r\in\Pi, H(r)=\atom}\mathit{BF}(r)$$
 where
 $$\mathit{BF}(r)\defi \bigwedge_{b \in
   B^+(r)}b \wedge \bigwedge_{c\in B^-(r)} \neg c.$$
 where, as usual, the conjunction for an empty set is understood as $\top$ and the empty disjunction as $\bot$.
An
interpretation~$I$ is a \emph{supported
  model}\longversion{~\cite{apt1988towards}} of
$\lp$, if it is a model of the formula~$\compl(\Pi)$.
Let $\SP$ be the set of all supported models of
$\Pi$. It holds that $\AS \subseteq
\SP$\longversion{~\cite{marek1992relationship}}, but not vice-versa.
If $\lp$ is tight, then $\AS =
\SP$\longversion{~\cite{cois1994consistency}}.
In practice, we use the completion in CNF, thereby introducing auxiliary variables
and still preserving the number of supported models.
\begin{example}\label{ex:running}
    
  Let $\lp_1 = \{a \leftarrow b; b \leftarrow; c \leftarrow c\}$.  
  We see that $\mathit{DP}(\Pi_1)$ is cyclic due to rule $c \leftarrow c$.
  Thus, $\lp_1$ is not
 tight and its respective answer sets $\AS[\lp_1] =
  \{\{a,b\}\}$ and supported
 models $\SP[\lp_1] = \{\{a,b\}, \{a,b,c\}\}$
  differ.
\end{example}

\paragraph{Assumptions.}
We define $\neg L \defi \{ \neg a \mid a \in L\}$ for a set~$L$ of
literals and assume that $\neg \neg a$ stands for $a$.  Let
$\lp$ be a program and $\mathcal{L}(\lp) \coloneqq \at(\lp) \cup \neg
\at(\lp)$ be its literals.  An \emph{assumption} is a literal~$\ell
\in \mathcal{L}(\Pi)$ interpreted as rule $\ic(\ell) \defi
\{\bot\leftarrow \pneg \ell\}$. %
For set~$L$ of assumptions of~$\lp$, we say that
$L$ is \emph{consistent}, if there is no atom~$a \in
L$ for which~$\neg a \in L$. Throughout this paper,
by~$L$ we refer to consistent assumptions.  Furthermore, we
define~$\ic(L) \defi \bigcup_{\ell \in L} \ic(\ell)$ and let~$\assume{\lp}{\ass}
\defi \lp \cup \ic(L)$.
\begin{example}
  Consider program~$\Pi_1$ from Example~\ref{ex:running},
  with
  $\AS[\Pi_1] = \{\{a,b\}\}$. For
  $\ass_1 \subseteq \{a, b, \pneg c\}$, we obtain the same answer
    sets,~i.e., $\mathcal{AS}(\Pi_1) = \AS[\assume{\Pi_1}{\ass_1}]$.
However, for any $\ass_2 \not\subseteq \{a, b, \pneg c\}$
we
  obtain $\AS[\assume{\Pi_1}{\ass_2}] = \emptyset$.
\end{example}

\paragraph{Knowledge
  Compilation and Counting on Formulas in
  sd-DNNF.} %
Let $\varphi$ be a formula,  $\varphi$ is in \emph{NNF (negation normal
  form)} if negations ($\neg$) occur only directly in front of
variables and the only other operators are conjunction ($\wedge$) and
disjunction ($\vee$)\longversion{~\cite{RobinsonVoronkov01}}.
NNFs can be represented in terms of {\em rooted directed acyclic
  graphs} (DAGs) where each leaf node is labeled with a literal, and
each internal node is labeled with either a conjunction
(\emph{$\wedge$-node}) or a disjunction (\emph{$\vee$-node}).

We use an NNF and its DAG interchangeably.
The \emph{size of an NNF}~$\varphi$, denoted by~$|\varphi|$, is given
by the number of edges in its DAG.
Formula~$\varphi$ is in \emph{DNNF}, if it is in NNF and it satisfies the
\emph{decomposability} property, that is, for any distinct subformulas
$\psi_i, \psi_j$ in a conjunction $\psi=\psi_1\land\dots \land \psi_n$
with $i\neq j$, we have
$\vars(\psi_i)\cap\vars(\psi_j)=\emptyset$\longversion{~\cite{Darwiche04a}}.
Formula $\varphi$ is in \emph{d-DNNF}, if it is in DNNF and it satisfies the
\emph{decision} property, that is, disjunctions are of the form
$\psi=(x\land \psi_1)\lor(\neg x\land \psi_2)$.
Note that $x$ does not occur in $\psi_1$ and $\psi_2$ because of
decomposability. $\psi_1$ and $\psi_2$ may be conjunctions.
Formula $\varphi$ is in \emph{sd-DNNF}, if all disjunctions in $\psi$ are
smooth, meaning for $\psi=\psi_1\lor \psi_2$ we have
$\vars(\psi_1) = \vars(\psi_2)$.

Determinism and smoothness permit traversal operations on sd-DNNFs to
count models of~$\varphi$ in linear time
in~$|\varphi|$\longversion{~\cite{darwiche2001tractable}}. The
traversal takes place on the so-called counting graph of an sd-DNNF.
The {\em counting graph}~$\mathcal{G}(\varphi)$ is the DAG of~$\varphi$ where
each node~$N$ is additionally labeled by $\val(N) \defi 1$, if $N$
consists of a literal; labeled by $\val(N) \defi \Sigma_{i} \val(N_i)$, if
$N$ is an $\vee$-node with children~$N_i$; labeled by
$\val(N) \defi \Pi_{i} \val(N_i)$, if $N$ is an $\wedge$-node.
By~$\val(\mathcal{G}(\varphi))$ we refer to~$\val(N)$ for the root~$N$
of~$\mathcal{G}(\varphi)$.
Function $\val$ can be constructed by traversing $\mathcal{G}(\varphi)$
in post-order in polynomial time.

It is well-known that $\val(\cg[\varphi])$ equals the model count
of~$\varphi$.  For a set~$L$ of literals, counting of
$\phi^L \defi \phi \wedge \bigwedge_{\ell \in L}\ell$ can be carried
out by \emph{conditioning} of~$\phi$
on~$L$~\cite{darwiche1999compiling}.  Therefore, the function~$\val$
on the counting graph is modified by setting $\val(N) = 0$, if~$N$
consists of~$\ell$ and~$\neg \ell \in L$.  This corresponds to
replacing each literal~$\ell$ of the NNF~$\varphi$ by constant~$\bot$
or $\top$, respectively. From now on, we denote by~$\sddcomplL$ an
equivalent sd-DNNF of~$\compl(\assume{\lp}{L})$ and its counting graph
by~$\CG$. Note that~$\assume{\lp}{L} = \lp$ for~$L = \emptyset$. The
conditioning of~$\mathcal{G}_\lp$ on~$L$ is denoted
by~$(\mathcal{G}_\lp)^L$.

\section{Counting Supported Models}%
In our applications mentioned in the introduction, we are interested
in counting multiple times under assumptions.
\myhighlight{In other words, we count the total number of answer sets and
  the number of answer sets under various changing assumptions.
}%
Therefore, we extend known techniques from knowledge
compilation~\cite{darwiche2002knowledge}.

The general outline for a given program~$\lp$ is as follows: (i)~we
construct the formula~$\compl(\lp)$ that can (ii)~be compiled in a
computationally expensive step into a formula~$\Phi_{\compl(\Pi)}$ in
a normal form, so-called sd-DNNF by existing knowledge
compilers. Then, (iii) on the sd-DNNF~$\Phi_{\compl(\Pi)}$ counting
can be done in polynomial time in the size of~$\Phi_{\compl(\Pi)}$.
We can even count under a set~$L$ of propositional assumptions by the
technique known as conditioning.

However, this approach yields only the number of supported models
under assumptions and we overcount compared to the number of answer
sets. To this end, in Section~\ref{sec:inc}, (iv)~we present a
technique to incrementally reduce the overcount.

In the following, we recall how knowledge compilation can be used
to count formulas under assumptions by assuming that a formula is in
sd-DNNF and constructing a counting graph.
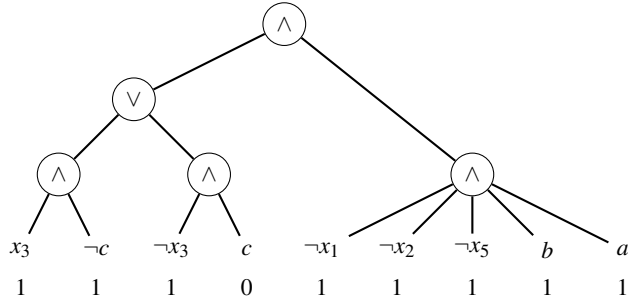
\begin{figure}[h]
\centering
\begin{tikzpicture}[
grow=right,
level distance=200mm,
sibling distance=105mm]
    \node[draw=none] (A) at (0,-1) {$x_3$};
    \node[draw=none] (B) at (1,-1){$\pneg c$};
    \node[draw=none] (C) at (2,-1){$\pneg x_3$};
    \node[draw=none] (D) at (3,-1) {$c$};
	\node[draw=none] (A_) at (0,-1.5) {\textcolor{black}{$1$}};
    \node[draw=none] (B_) at (1,-1.5){\textcolor{black}{$1$}};
    \node[draw=none] (C_) at (2,-1.5){\textcolor{black}{$1$}};
    \node[draw=none] (D_) at (3,-1.5) {\textcolor{black}{$0$}};

    \node[shape=circle,draw=black,fill=white] (E) at (0.5,0) {$\wedge$};
    \node[shape=circle,draw=black,fill=white] (F) at (2.5,0) {$\wedge$};
    \node[draw=none] (G) at (4,-1) {$\pneg x_1$};
    \node[draw=none] (H) at (5,-1) {$\pneg x_2$};
    \node[draw=none] (I) at (6,-1) {$\pneg x_5$};
    \node[draw=none] (J) at (7,-1) {$b$};
    \node[draw=none] (K) at (8,-1) {$a$};
    \node[draw=none] (G_) at (4,-1.5) {\textcolor{black}{$1$}};
    \node[draw=none] (H_) at (5,-1.5) {\textcolor{black}{$1$}};
    \node[draw=none] (I_) at (6,-1.5) {\textcolor{black}{$1$}};
    \node[draw=none] (J_) at (7,-1.5) {\textcolor{black}{$1$}};
    \node[draw=none] (K_) at (8,-1.5) {\textcolor{black}{$1$}};
    \node[shape=circle,draw=black,fill=white] (L) at (6,0) {$\wedge$};
    \node[shape=circle,draw=black,fill=white] (M) at (1.5,1) {$\vee$};
    \node[shape=circle,draw=black,fill=white] (N) at (3.5,2) {$\wedge$};

    \path [-,thick](A) edge node[right] {} (E);
    \path [-,thick](B) edge node[right] {} (E);
    \path [-,thick](C) edge node[right] {} (F);
    \path [-,thick](D) edge node[right] {} (F);
    \path [-,thick](G) edge node[right] {} (L);
    \path [-,thick](H) edge node[right] {} (L);
    \path [-,thick](I) edge node[right] {} (L);
    \path [-,thick](J) edge node[right] {} (L);
    \path [-,thick](K) edge node[right] {} (L);

    \path [-,thick](E) edge node[right] {} (M);
    \path [-,thick](F) edge node[right] {} (M);

    \path [-,thick](M) edge node[right] {} (N);
    \path [-,thick](L) edge node[right] {} (N);

\end{tikzpicture}
\caption{Counting graph~$\mathcal{G}(\varphi \wedge \neg c)$
	labeled with literals and their respective value.
}
\label{fig:cg}
\end{figure}
\begin{example}\label{ex:3}
  Consider the
  sd-DNNF~$\varphi_1=((x_3 \land \neg c) \lor (\neg x_3 \land c))
  \land (\neg x_1 \land \neg x_2 \land \neg x_5 \land a \land b)$.  We
  observe in Figure~\ref{fig:cg} that its rooted directed acyclic
  graph (DAG) has $14$~nodes, $7$~variables, and~$13$~edges.
  In consequence, we have that $|\varphi_1|=13$.  By conditioning of $\varphi$ on $L=\{\neg c\}$,
  each variable in $L$ will be removed from~$\cg[\varphi_1]$ and we obtain
  $\varphi_1 \wedge \neg c = ((x_3 \land \neg \bot) \lor (\neg x_3
  \land \bot)) \land (\neg x_1 \land \neg x_2 \land \neg x_5 \land a
  \land b)$.  From Figure~\ref{fig:cg}, we observe that the model
  count~$\val(\cg[\varphi\wedge \neg c])$ of
  formula~$\varphi\wedge \neg c$ is~$1$.
\end{example}
Using the techniques as described above, we can compile the
formula~$\compl(\lp)$ into an sd-DNNF~$\Phi_{\compl(\lp)}$ and count the
number~$\Card{\SP}$ of supported models. We illustrate this in the
following example.
\begin{example}\label{ex:prog}
  Consider $\lp_1$ from Example~\ref{ex:running}.
  When constructing~$\compl(\lp_1)$ in CNF, we obtain~$10$ clauses with~$4$ new
  auxiliary variables~$x_1$, $x_2$, $x_3$, and $x_5$.
  We can compile it into an sd-DNNF~$\Phi_{\lp_1}$ which
  is logically equivalent to~$\compl(\lp_1)$. For illustration
  purposes, we chose formula~$\varphi_1$ from Example~\ref{ex:3} such that
  $\Phi_{\lp_1}$ is equivalent to $\varphi_1$.
  Hence, we can obtain the number~$\Card{\SMod(\Pi_1)}$ of supported
	models from~$\val(\mathcal{G}_{\lp_1})$.
\end{example}

\subsection{Counting Supported Models under Assumptions}
\myhighlight{Since assumptions of formulas and programs behave slightly
  differently due to the GL reduct}, it is not immediately clear that
we can use conditioning to obtain the number of supported models of a
program under given assumptions.  In the following we will show that
supported models of~$\lp$ under assumptions~$\ass$ coincide with
models of~$\sddcomplL$.

\begin{observation}\label{lem:models}
  Let $\lp$ be a program and $\ass$ assumptions. Then,
  $\Mod(\sddcomplL) = \SP[\assume{\lp}{\ass}]$
\end{observation}
For any program~$\lp$ the conditioning~$(\Phi_\Pi)^L$ on 
assumptions~$\ass$
allows us to identify supported models of a program~$\assume{\lp}{\ass}$.
\begin{lemma}
  Let $\lp$ be a program and $L$ be
  assumptions. Then,
  $\Mod({(\sddcompl)^{L}})=\SMod(\assume{\lp}{\ass})$.
\end{lemma}
\begin{proof}

We first establish the following claim:
\begin{align}\label{eqn:compl-ass}
  \compl(\assume{\Pi}{L})=\compl(\Pi\cup\ic(L)) = %
  \compl(\Pi)\wedge \bigwedge_{\ell\in L} \ell
\end{align}
\noindent By definition, we have that $\compl(\assume{\Pi}{L})=\compl(\Pi\cup \ic(L))$.
This further evaluates to $\compl(\Pi)\cup \ic(L)$.
Since $\bot$ evaluates to false always and
\[\compl(\{\bot \leftarrow B(r)^+,\neg B(r)^- \mid r\in \lp,
  H(r)=\bot\}) = \bot \leftrightarrow \bigvee_{r\in \lp, H(r)=\bot}
  \mathit{BF}(r),\]
we obtain that
\begin{align}
	\Mod(\bot \leftrightarrow \bigvee_{r\in \lp, H(r)=\bot} \mathit{BF}(r)) =&
                                                                            \Mod(\bigwedge_{r\in \lp, H(r)=\bot} \bot\leftrightarrow\mathit{BF}(r)),\\
 = & \Mod(\bigwedge_{r\in \lp, H(r)=\bot} \neg\mathit{BF}(r)).
\end{align}
As a result,
\begin{align}
	\Mod(\compl(\assume{\Pi}{L}\setminus\ic(L))\cup \ic(L))=&\Mod(\compl(\assume{\Pi}{L}\setminus\ic(L))\cup\bigcup_{\ell\in
                                                  L}\compl(\ic(\ell))\\
  =&\Mod(\compl(\Pi)\wedge \bigwedge_{\ell\in
     L}\compl(\ic(\ell)))\\
  =&\Mod(\compl(\Pi)\wedge \bigwedge_{\ell\in L}\neg
     \mathit{BF}(\ic(\ell)))\\
  =&\Mod(\compl(\Pi)\wedge \bigwedge_{\ell\in L}
     \ell).
\end{align}
In consequence, Equation~\ref{eqn:compl-ass} holds.
It remains to show that conditioning~$(\Phi_\Pi)^L$ in the
sd-DNNF~$\Phi_\Pi$ preserves all models according to $\Pi$ under the
set~$L$ of assumptions.
By definition of conditioning, it holds that
$\Mod((\Phi_\Pi)^L)= \Mod(\Phi_\Pi \wedge \bigwedge_{\ell\in L}\ell)$.
By assumption, it is true that
$\Mod(\Phi_\Pi \wedge \bigwedge_{\ell\in L}\ell) =
\Mod(\compl(\Pi)\wedge \bigwedge_{\ell\in L}\ell)$.
From Equation~\ref{eqn:compl-ass}, we obtain that
$\Mod(\compl(\Pi)\wedge \bigwedge_{\ell\in L}\ell) =
\Mod(\compl(\assume{\Pi}{L}))$.  By definition,
$\Mod(\compl(\assume{\Pi}{L})) = \SP[\assume{\Pi}{L}]$.
In consequence, we established that
    $\Mod({(\Phi_{\lp})^{L}}) = \SP[\assume{\Pi}{L}]$.  Hence, the Lemma sustains.
\hfill
\end{proof}

Immediately, we obtain that we can count the number of supported
models by first compiling the completion into an sd-DNNF and then
applying conditioning. For tight programs, this already yields the
number of answer sets.
\begin{corollary}\label{corr:spcount}
  Let $\lp$ be a program and $\ass$ be assumptions. Then,
  \begin{align*}
  \val((\mathcal{G}_\Pi)^L) = \Card{\Mod((\Phi_\Pi)^L)} =
    \Card{\SP[\assume{\lp}{\ass}]}.
  \end{align*}
  If
  $\lp$ is tight, also $\val((\mathcal{G}_\Pi)^L) =
  \Card{\AS[\assume{\lp}{\ass}]}$ holds.
  Furthermore, counting can be done in time linear in~$|\sd|$.
\end{corollary}

\begin{example}
  Consider program~$\Pi_1$ from Example~\ref{ex:running}, which has two supported models~$\{a,b\}$ and $\{a,b,c\}$.
  Without setting $\val(c)$ to~$0$ in Figure~\ref{fig:cg}, we would
  obtain~$2$, which corresponds to these two models.
  By assumption~$\neg c$, we set $\val(c)$ to~$0$, which results in a
  total count of~$1$ as the $\wedge$-node gives only one count in
  the subgraph.
\end{example}

\subsection{Compressing Counting Graphs}
When computing the counting graph of the completion of a
program~$\Pi$, \myhighlight{%
  in practice, we usually construct a CNF of the completion by
  introducing so-called nogoods~\cite{GebserKaufmannSchaub12a} similar
  to Tseitin's transformation~\shortcite{Tseitin83}.
}%
It is well-known that there is a one-to-one correspondence, however,
auxiliary variables are introduced,
\myhighlight{see,~e.g.,~\cite{KuiterKrieterSundermann23a}.}  For counting,
the one-to-one correspondence immediately allows to establish a
bijection between the models of the CNF and the supported models
making it practicable on CNFs.

However, from Corollary~\ref{corr:spcount}, we know that the runtime
counting models on~$(\mathcal{G}_\Pi)^L$ depends on the size of~$\sd$.
In consequence, introducing auxiliary variables affects the runtime of
our approach.
To this end, we introduce a compressing technique in
Algorithm~\ref{alg:ccgstd} that takes a counting
graph~$\mathcal{G}_\Pi$ and produces a \emph{compressed counting
  graph} (CCG)~$\tau(\mathcal{G}_\Pi)$, thereby removing auxiliary variables that
have been introduced by the Tseitin transformation. %
\begin{algorithm}[t]
  \textbf{In}: sd-DNNF $\sd$, $\A$\\
  \textbf{Out}: Compressed counting graph $\tau(\mathcal{G}_\Pi)$
\begin{algorithmic}[1] %
\STATE initialize array~$\mathtt{t}$ and traverse nodes~$N \in \sd$ bottom-up such that \;
\STATE \hspace{0.3cm}\textbf{if} $N$ contains a literal~$\ell \in \mathcal{L}(\Pi)$
	\textbf{then} label~$N$ with~$\val(N)$\;
\STATE \hspace{0.3cm}\textbf{else if} $N$ contains a literal~$\ell \notin \mathcal{L}(\Pi)$ \textbf{then} mark~$N$ as~$\mathtt{ignored}$\label{alg:aux}\;
\STATE 
\hspace{0.3cm}\textbf{else} check the number of children of~$N$ that are not marked as~$\mathtt{ignored}$\;
\STATE 
\hspace{0.6cm}\textbf{if} $N$ has no remaining children \textbf{then} mark~$N$ as~$\mathtt{ignored}$\label{alg:case:start} \label{alg:case:nochild}\;
\STATE 
\hspace{0.6cm}\textbf{else if} $N$ has one remaining child~$C$ \textbf{then}~$N
	\gets C$ and mark~$N$ as~$\mathtt{ignored}$ \label{alg:case:absorb}\; %
\STATE 
\hspace{0.6cm}\textbf{else}~$v \gets \val(N)$ w.r.t.~$\mathtt{t}$ and remaining children of ~$N$ and label~$N$ with~$v$\label{alg:case:end}\;
\STATE 
\hspace{0.3cm} add~$N$ to~$\mathtt{t}$\label{alg:add}\;
\STATE 
remove all nodes marked with~$\mathtt{ignored}$ from~$\mathtt{t}$\label{alg:remove}\;
\STATE 
\textbf{return}~$\mathtt{t}$\;
\end{algorithmic}
\caption{Counting Graph Compression}
\label{alg:ccgstd}
\end{algorithm}
The algorithm %
takes as input an sd-DNNF $\sddcompl$,
and literals $\A$; and returns the
compressed counting graph $\tau(\mathcal{G}_\Pi)$.  In
Line~\ref{alg:aux}, we check whether the literal node consists of an
auxiliary variable, and if so, it will be ignored.
The case distinction in Lines~\ref{alg:case:start}--\ref{alg:case:end}
distinguishes how many not ignored children a non-literal node still
has. Remember that each non-literal node is either an $\wedge$-node or
an $\vee$-node.  In Line~\ref{alg:case:nochild}, the node can be
removed, as it has no child.  In Line~\ref{alg:case:absorb}, the node
needs to be absorbed, as it has only one child meaning that the
node ultimately becomes its child.
In all other cases (Line~\ref{alg:case:end}), the node needs to be
evaluated on the CCG~$\mathtt{t}$
such that the ignored nodes are treated as neutral element of the
respective sum or product.  Ignored nodes are then removed
from~$\mathtt{t}$.
It remains to show that compressing~$\mathcal{G}_\Pi$ leaves $\val$ unchanged, which is the topic of the following statement and subsequent proof.
\begin{lemma}\label{thm:stdccg}
  Let~$\lp$ be a program, 
  $\Phi_\Pi$ an sd-DNNF of~$\compl(\Pi)$ after a transformation that
  preserves the number of models, but introduces auxiliary variables,
  and $\mathcal{G}_\Pi$ its counting graph.  %
  Then, $\val(\tau(\mathcal{G}_\Pi)) = \val(\mathcal{G}_\Pi)$
	and $\tau(\mathcal{G}_\lp)$ can be constructed in time~$\mathcal{O}(2 \cdot
  |\sd|)$.
\end{lemma}
\begin{proof}
  Let
  $\mathcal{G}_\Pi$ be the counting graph of an sd-DNNF that is
  equivalent to the CNF that has been constructed
  from~$\compl(\Pi)$ using a transformation that preserves the number
  of models, which usually is the Tseitin transformation.
  We show that the value~$\val(N)$ of each
  node~$N$ of
  $\mathcal{G}_\Pi$, which is not removed
  in~$\tau(\mathcal{G}_\Pi)$, does not change, since
  for~$N$ and its respective
  children~$\mathit{children}(N)$ in Algorithm~\ref{alg:ccgstd} we
  modify only literals that occur in the program~$\Pi$.
  By $N_{\tau} \in
  \tau(\mathcal{G}_\Pi)$ we denote the modified version
  of~$N$, and by
  $\mathit{children}_{\tau}(N)$ we denote the children
  of~$N$ in
  $\tau(\mathcal{G}_\Pi)$. 
  We distinguish the cases:
  \begin{enumerate}
  \item Suppose $N$ is a literal node. Let
    $\ell$ denote the corresponding literal.  If $\ell \not \in
    \A$, then $N$ is removed in
    $\tau(\mathcal{G}_\Pi)$, thus by contraposition, we know that, if
    $N$ is not removed in $\tau(\mathcal{G}_\Pi)$, then $\ell \in
    \A$.  Assume $\ell \in \A$. Then $N =
    N_{\tau}$. Therefore, $\val(N) = \val(N_{\tau}) \in \{0,1\}$.

  \item Suppose $N$ is not a literal node. Then, since $N$ is an $\wedge$- or an
  $\vee$-node, we know that $|\mathit{children}(N)| \geq 2$.  However, in general
  $0 \leq |\mathit{children}_{\tau}(N)| \leq |\mathit{children}(N)|$.
  \begin{enumerate}
  \item\label{case:0} Assume $|\mathit{children}_{\tau}(N)| =
    0$. Then, in Algorithm~\ref{alg:ccgstd},
    $N$ will be ignored and thus not belong to
    $\tau(\mathcal{G}_\Pi)$.
  \item\label{case:1} Assume $|\mathit{children}_{\tau}(N)| =
    1$. Then, in Algorithm~\ref{alg:ccgstd},
    $N$ will be absorbed by its only child. Thus,
    $N$ does not belong to $\tau(\mathcal{G}_\Pi)$.
  \item\label{case:2} Assume $|\mathit{children}_{\tau}(N)| \geq
    2$. Then in Algorithm~\ref{alg:ccgstd},
    $N$ will be evaluated on
    $\mathit{children}_{\tau}(N)$, which means
    $N_{\tau}$ will be contained in $\tau(\mathcal{G}_\Pi)$.
    We now need to show that $\val(N)$ on\linebreak
    $\mathit{children}(N)$ corresponds to
    $\val(N)$ on $\mathit{children}_{\tau}(N)$,~i.e., $\val(N) =
    \val(N_{\tau})$. By assumption (number of models is preserved), we
    have a bijection between $M(\Phi_\Pi)$ and
    $\SP[\lp]$ which ignores auxiliary variables.  
    Therefore, we can simply set the values of children
    $\mathit{children}(N)$ that have been removed or absorbed due to
    Cases~\ref{case:0},~\ref{case:1}, or~\ref{case:2} -- as a
    consequence of removing auxiliary variables -- to the
    corresponding neutral element of the value of $N$.
    \begin{enumerate}
      \item Assume $N$ is an $\wedge$-node.
      Accordingly, in Algorithm~\ref{alg:ccgstd}, $N$
      will be evaluated on\linebreak $\mathit{children}_{\tau}(N)$ such that in the
      product corresponding to $\val(N)$, the value of each
      removed branch (removed child), due to removing auxiliary variables,
      corresponds to the neutral element of multiplication, i.e.,~$1$.
      Therefore, we conclude that $\val(N) = \val(N_{\tau})$. 
      \item Assume $N$ is an $\vee$-node. 
      Again, accordingly, in
      Algorithm~\ref{alg:ccgstd}, $N$ will be evaluated on
      $\mathit{children}_{\tau}(N)$ such that in the sum corresponding to
      $\val(N)$, the value of each removed branch (removed child), due
      to removing auxiliary variables, corresponds to the neutral element of
      addition, i.e.,~$0$. Therefore, $\val(N) =
      \val(N_{\tau})$, which concludes the proof.
    \end{enumerate}
    
  \end{enumerate}
  \end{enumerate} 

  \noindent Inspecting Algorithm~\ref{alg:ccgstd}, we see that we
  require two traversals of the original counting graph, one from
  Lines~\ref{alg:aux}--\ref{alg:add} and another one in
  Line~\ref{alg:remove} where we remove the nodes that do not belong
  to the CCG.
  Runtime follows from the fact that we need to traverse~$\sd$ twice.
\hfill
\end{proof}

\begin{corollary}
  Let $\lp$ be a tight program, then~$\val(\tau(\mathcal{G}_\Pi)) = |\AS[\lp]|$.
\end{corollary}

\section{Incremental Counting by Inclusion-Exclusion}\label{sec:inc}
In the previous section, we illustrated how counting on tight programs
works and introduced a technique to speed up practical counting.
To count answer sets of a non-tight program, we need to distinguish
supported models from answer sets
on~$\tau(\mathcal{G}_\Pi)$, which can become quite tedious.
Therefore, we use the positive dependency graph
$\mathit{DP}(\lp)$ of $\lp$. A set $X \subseteq
\at(\lp)$ of atoms is an answer set, whenever it can be derived from
$\lp$ in a finite number of steps. In particular, the mismatch between
answer sets and supported models is caused by atoms~$C \in
\mathit{cycles}(\lp)$ involved in cycles
in $\mathit{DP}(\lp)$ that are not supported by atoms from outside the cycle. We
call those supporting atoms of $C$ the \emph{external support} of~$C$.

\begin{definition}\label{def:es}
  Let $\lp$ be a program and $r \in \lp$. 
  An atom~$a \in B^+(r)$ %
  is an \emph{external support}
  of $C \in \mathit{cycles}(\lp)$, whenever
  $H(r) \subseteq C$ %
	and $B^+(r) \cap C = \emptyset$. %
  By $\mathit{ES}(C)$ we denote the set of all external supports of~$C$.
\end{definition}

\noindent Next, we illustrate the effect of external supports on the answer sets derivation.
\begin{example}
  Let $\lp_2 = \{a \leftarrow b;
    b \leftarrow a;
    a \leftarrow c;
    c \leftarrow \pneg d;
    d \leftarrow \pneg c\}$.
    The positive dependency graph of $\lp_2$ is given in Figure~\ref{fig:cyc}.
  We obtain a cycle $C = \{a,b\}$ due to rules $a \leftarrow b$ and $b
  \leftarrow a$ with 
 external support $\mathit{ES}(C) = \{c\}$ due to rule $a
  \leftarrow c.$  
  However,
 due to rules $c \leftarrow \pneg d$ and $d
  \leftarrow \pneg c$, we see that whenever $d$ is true,
 $c$ is false, so
  that $d$ deactivates the support of $C$, which means that
 $\{a,b,d\}$ cannot
  be derived from $\lp_2$ in a finite number of steps.
 Accordingly, we have
  $\SP[\lp_2] = \{\{a,b,c\}, \{a,b,d\}, \{d\}\}$, but
 $\AS[\lp_2] =
  \{\{a,b,c\}, \{d\}\}$.
\end{example}
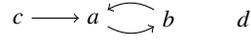
\begin{figure}[t]
\centering
\bigskip
\begin{tikzpicture}[
grow=right,
level distance=200mm,
sibling distance=105mm]
    \node[draw=none] (A) at (0,-1) {$c$};
    \node[draw=none] (B) at (1,-1){$a$};
    \node[draw=none] (C) at (2,-1){$b$};
    \node[draw=none] (D) at (3,-1) {$d$};

    \path [->](A) edge node[right] {} (B);
    \path [->, bend right](B) edge node[right] {} (C);
    \path [->, bend right](C) edge node[right] {} (B);

\end{tikzpicture}
\caption{The positive dependency graph of~$\lp_2$.}
\label{fig:cyc}
\end{figure}
Note that external supports are sets of
atoms. However, we can simulate such a set
by introducing an auxiliary atom; hence one atom,
as in this definition, is
sufficient~\cite{GebserKaufmannSchaub12a}.
\begin{example}
  Let $a \leftarrow b$, $b \leftarrow a$, and $b \leftarrow c, \pneg
  d$ be rules.  Then the external support of atoms
  $\{a,b\}$, which are involved in cycles, is $\{c\}$. If instead of $b \leftarrow c,
  \pneg d$ we use two alternative rules $b_r \leftarrow c, \pneg
  d$ and $b \leftarrow b_r$, we have $\mathit{ES}(\{a,b\}) = \{b_r\}$.%
\end{example}

To approach the answer set count of a non-tight program under
assumptions, we employ the well-known \emph{inclusion-exclusion principle},
which is a counting technique to determine the number of elements in a finite
union of finite sets $X_1, \dots, X_n$. 
Therefore, first the cardinalities of the singletons are summed up.
Then, to compensate for potential overcounting, the cardinalities of
all intersections of two sets are subtracted. Next, the number of
elements that appear in at least three sets are added back,~i.e., the
cardinality of the intersection of all three sets -- to compensate for
potential undercounting -- and so on.
As an example, for three sets $X_1, X_2, X_3$ the procedure can be expressed as
$|X_1 \cup X_2 \cup X_3| = |X_1| + |X_2| + |X_3| - |X_1 \cap X_2| - |X_1 \cap X_3| - |X_2 \cap X_3|
  + |X_1 \cap X_2 \cap X_3|$.
This principle can be used to count answer sets via %
supported model counting. 

Next we define a notion that is useful to identify or prune supported models 
that are not stable.
\begin{definition}
We define the \emph{unsupported constraint} for a set $C = \{c_0,
\dots, c_n\} \in \mathit{cycles}(\lp)$ of atoms involved in cycles and its respective\
external supports $\mathit{ES}(C) = \{ s_0, \dots, s_m\}$ by the rule $\lambda(C) \defi
\bot \leftarrow c_0, \dots, c_n, \pneg s_0, \dots, \pneg s_m$.
\end{definition}
\myhighlight{The unsupported constraints as defined here, (i) are inspired by \emph{loop formulas}~\cite{LinZ04,ferraris2006generalization}; and (ii) contain the whole set~$C$, which is
slightly weaker than constraints (nogoods) defined in related
work~\cite{GebserKaufmannSchaub12a}, but sufficient for %
characterizing answer sets.} %
\begin{lemma}\label{obs:cycles}
  Let $\lp$ be a program with cycles~$\mathit{cycles}(\lp) = \{C_1,\dots,C_n\}$, then
  \begin{align*}
    \AS[\Pi] = \SP[\lp \cup \{ \lambda(C_1), \dots, \lambda(C_n) \} ].
  \end{align*}
\end{lemma}%
\begin{proof}%
  Recall that $\AS[\Pi]\subseteq \SP[\Pi]$.  However, supported models
  -- in particular those that are not answer sets -- might contain a
  cycle~$C = \{c_0, \dots, c_m\} \in \mathit{cycles}(\lp)$ without
  external support from $\mathit{ES}(C) = \{ s_0, \dots, s_k\}$, which
  are precisely those supported models we exclude by adding a
  rule
  \[ \bot \leftarrow c_0, \dots, c_m, \pneg s_0, \dots, \pneg s_k \]
  in the form of unsupported constraints $\lambda(C)$ to $\lp$ for
  each~$C \in \mathit{cycles}(\lp)$.  This ensures that atoms involved in cycles
  are not present without external support in any supported model,
  which provides us with supported models that are answer sets.
\end{proof}

\begin{example}\label{ex:cycles}
  Let
  $\lp_3 = \lp_2 \cup \{b \leftarrow g; f \leftarrow g; e \leftarrow
  f; f \leftarrow e\}$, which has two cycles $C_0 = \{a,b\}$ and
  $C_1 = \{e,f\}$.  Their corresponding external supports are
  $\mathit{ES}(C_0) = \{c,g\}$ and $\mathit{ES}(C_1) =
  \{g\}$. Accordingly, we have unsupported
  constraints~$\lambda(C_0) = \bot \leftarrow a,b,\pneg c,\pneg g$
  and~$\lambda(C_1) = \bot \leftarrow e,f,\pneg g$.
  Figure~\ref{fig:cyc3} illustrates the positive dependency graph of
  program~$\lp_3$.
\end{example}

\begin{figure}[t]
  \centering
  \begin{tikzpicture}[grow=right,level distance=200mm,sibling distance=105mm]
    \node[draw=none] (A) at (0,-1) {$c$};
    \node[draw=none] (B) at (1,-1){$a$};
    \node[draw=none] (C) at (2,-1){$b$};
    \node[draw=none] (H) at (3,-1) {$g$};
    \node[draw=none] (G) at (4,-1){$e$};
    \node[draw=none] (F) at (5,-1){$f$};
    \node[draw=none] (D) at (6,-1) {$d$};

    \path [->](A) edge node[right] {} (B);
    \path [->, bend right](B) edge node[right] {} (C);
    \path [->, bend right](C) edge node[right] {} (B);
    \path [->](H) edge node[right] {} (C);

    \path [->](H) edge node[right] {} (G);

    \path [->, bend right](G) edge node[right] {} (F);
    \path [->, bend right](F) edge node[right] {} (G);
  \end{tikzpicture}
  \caption{The positive dependency graph of~$\lp_3$ from Example~\ref{ex:cycles}.}
  \label{fig:cyc3}
\end{figure}
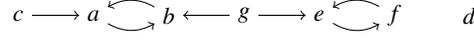

Before we discuss our approach on incremental answer set counting, we
need some further notation.  From now on, by
$\Lambda_{d}(\lp) \defi \{\{\lambda(C_1), \dots, \lambda(C_d)\} \mid
\{C_1, \dots, C_d\} \subseteq \mathit{cycles}(\lp)\}$ we denote the
set of all combinations of unsupported constraints of cycles that
occur in any subset of $\mathit{cycles}(\lp)$ with cardinality
$0\leq d\leq n$, where $n \defi \Card{\mathit{cycles}(\lp)}$.
Further, we define body literals of a set of unsupported
contraints~$\Gamma$
by~$B(\Gamma) \defi \bigcup \{ B(\lambda(C)) \mid \lambda(C) \in
\Gamma\}$.
\begin{example}[Continued]
  Consider program~$\lp_3$ from Example~\ref{ex:cycles}.  We
  have~$\Lambda_0(\lp_3) =
  \emptyset$,~$\Lambda_1(\lp_3) =
  \{\{\lambda(C_0)\},\{\lambda(C_1)\}\}$
  and~$\Lambda_2(\lp_3) = \{\{\lambda(C_0), \lambda(C_1)\}\}$.
\end{example}
\noindent
Now, we define the \emph{incremental count} of
$\Card{\AS[\assume{\lp}{\ass}]}$ by $a_d^\ass$, using the combinatorial
principle of inclusion-exclusion as follows:

\begin{align}
  a_d^\ass &\defi \sum_{i=0}^{d} (-1)^i \sum_{\Gamma \in \Lambda_{i}(\lp)}
             \Card{\SP[\assume{\lp}{{\ass \cup B(\Gamma)}}]} \\
           &=
             \Card{\SP[\assume{\lp}{\ass}]} - \sum_{\Gamma \in \Lambda_{1}(\lp)}
             \Card{\SP[\assume{\lp}{\ass \cup B(\Gamma)}]}\\
           &+ \sum_{\Gamma \in \Lambda_{2}(\lp)} |\SP[\assume{\lp}{\ass \cup B(\Gamma)}]|
             - \dots + (-1)^d \sum_{\Gamma \in \Lambda_{d}(\lp)} \Card{\SP[\assume{\lp}{\ass \cup B(\Gamma)}]}
\end{align}

\noindent By subtracting $\Card{\SP[\assume{\lp}{L}] \setminus \SP[\assume{\lp}{\ass \cup B(\Gamma)}]}$ for each $\Gamma \in
\Lambda_{1}(\lp)$ we subtract the number of supported models that are \emph{not
answer sets} under assumptions $\ass$ with respect to each cycle $C \in
\mathit{cycles}(\lp)$.  However, we need to take into account the interaction of
cycles and their respective external supports under assumptions $\ass$. Thus we
enter the first alternation step, where we proceed by adding back
$|\SP[\assume{\lp}{L}] \setminus  \SP[\assume{\lp}{\ass \cup B(\Gamma)}]|$ for each $\Gamma \in \Lambda_{2}(\lp)$, which
means that we add back the number of supported models that were mistakenly
subtracted from $\Card{\SP[\assume{\lp}{L}]}$ in the previous step, and so on, until we
went through all $\Lambda_i$ where $0 \leq i \leq d$. Note that therefore in
total we have $d$ alternations.
In general, we %
show that~$a^L_{n}=|\AS[\assume{\Pi}{L}]|$ as follows.
\begin{theorem}\label{prop:incl}
  Let $\Pi$ be a program,
  $\mathit{cycles}(\lp) = \{C_1, \dots, C_n\}$, and further
  $U \defi \{\lambda(C_1), \dots, \lambda(C_n)\}$ be the set of all
  unsupported constraints of $\lp$.  Then, for assumptions~$L$,
  \[|\SP[\assume{\lp}{\ass}\cup U]| = \sum_{i=0}^{n} (-1)^i
    \sum_{\Gamma \in \Lambda_{i}(\lp)} |\SP[\assume{\lp}{L}] \setminus
    \SP[\assume{\lp}{\ass \cup B(\Gamma)}]|\]
\end{theorem}
\begin{proof}%
We proceed by induction on~$|\mathit{cycles}(\Pi)|$.
\paragraph{Induction Base Case:} We assume that $|\mathit{cycles}(\Pi)|=0$.
Then, since~$\Pi$ admits no positive cycle in~$\mathit{DP}(\Pi)$, 
we have~$\AS[\assume{\Pi}{L}]=\SP[\assume{\Pi}{L}]$, and therefore~$|\AS[\assume{\Pi}{L}]|=|\SP[\assume{\Pi}{L}]|$.

	\paragraph{Induction Hypothesis (IH):} We assume that the proposition holds for
every program~$\Pi$ with a number of cycles~$|\mathit{cycles}(\Pi)|<m$.

\paragraph{Induction Step:} 
We need to show that the result holds for a program~$\Pi$
with~$|\mathit{cycles}(\Pi)|=m+1$.
Let~$C'\in \mathit{cycles}(\Pi)$ be a cycle. 
We define $U_m \defi \{\lambda(C_1), \dots, \lambda(C_m)\}$ for any~$\{C_1, \ldots, C_m\}\subseteq \mathit{cycles}(\Pi)$ such that~$|U_m| = m$ with $C_i\neq C'$ for~$C_i\in \{C_1, \ldots, C_m\}$.
Then, by IH, we have that 
	\[x \defi |\SP[\assume{\lp}{\ass \cup B(U_m)}]| = \sum_{i=0}^{m} (-1)^i \sum_{\Gamma
	\in \Lambda_{i}(\lp), \lambda(C')\notin\Gamma} |\SP[\assume{\lp}{L}] \setminus
	\SP[\assume{\lp}{\ass \cup B(\Gamma)}]|\]
To $x$, %
the formula $\sum_{i=0}^{m+1} (-1)^i \sum_{\Gamma \in \Lambda_{i}(\lp)}
	|\SP[\assume{\lp}{L}] \setminus \SP[\assume{\lp}{\ass \cup B(\Gamma)}]|$ adds
$|\SP[\lp \cup \lambda(C')]|$. 
However, this formula then subtracts supported models satisfying
both constraints $\{\lambda(C'), \lambda(C'')\}$ with one of the cycles $\lambda(C'') \in U_m$ twice, %
which require to be added back. Thus, we proceed by adding back supported models satisfying unsupported constraints 
of $C'$ with two other cycles, which again have to be subtracted in the next step. In turn, the application 
of the inclusion-exclusion principle ensures that
\begin{align*}
  \sum_{i=0}^{m+1} (-1)^i
  \sum_{\Gamma \in \Lambda_{i}(\lp)} |\SP[\assume{\lp}{L}] \setminus \SP[\assume{\lp}{\ass
  \cup B(\Gamma)}]|\\
  =
  x +
  \sum_{i=0}^{m+1} (-1)^i \sum_{\Gamma \in \Lambda_{i}(\lp), \lambda(C')\in
  \Gamma} |\SP[\assume{\lp}{L}] \setminus \SP[\assume{\lp}{\ass \cup B(\Gamma)}]|.
\end{align*}

\hfill 
\end{proof}
Finally, one can count answer sets correctly.
\begin{corollary}\label{corr:ae}
Let $\Pi$ be a program, $L$ assumptions, and $n = |\mathit{cycles}(\lp)|$.
Then, $a_n^\ass = |\AS[\assume{\Pi}{L}]|$.
\end{corollary}
In fact, we can characterize~$a^{\ass}_{n}$ with respect to alternation depths.
If there is no change from one alternation to another, the point is reached where the
number of answer sets is obtained, as the following lemma states.
\begin{lemma}\label{thm:term}
  Let~$\lp$ be a program and~$L$ be assumptions. If $a^{\ass}_i =
  a^{\ass}_{i+1}$ for some integer~$i\geq 0$, then $a^{\ass}_i =
    |\mathcal{AS}(\assume{\lp}{\ass})|$.
\end{lemma}
\begin{proof}

  Suppose $a^{\ass}_i = a^{\ass}_{i+1}$, then 
  \(\sum_{\Gamma \in \Lambda_{i+1}(\lp)} |\SP[\assume{\lp}{L}] \setminus \SP[\assume{\lp}{\ass
	\cup B(\Gamma)}]| = 0\). We can observe that therefore no further combination of unsupported constraints with set~$\ass$ of assumptions
  where we combine unsupported constraints of cycles that occur in subsets of $\mathit{cycles}(\lp)$
  with cardinality $j > i+1$
  points to any supported model. In other words, we have for all $j > i$ that
  \(\sum_{\Gamma \in \Lambda_{j}(\lp)}
	|\SP[\assume{\lp}{L}] \setminus \SP[\assume{\lp}{\ass \cup B(\Gamma)}]| = 0\),
  which concludes the proof. 
\hfill
\end{proof}

\begin{algorithm}[t]
\textbf{In}: Program~$\Pi$; assumptions~$L$; compressed counting graph
	$\tau(\mathcal{G}_\Pi)$; alternation depth $d$\\ 
\textbf{Out}: Incremental count $a_d^{\ass}$%
\begin{algorithmic}[1] %
\STATE 
$\mathtt{count} \gets \val(\tau(\mathcal{G}_\Pi)^L$) and~$c \gets 0$ %
\STATE \textbf{if}~$d$ is odd \textbf{then}~$d \gets d + 1$\; \label{alg:case:noc}
\STATE for every~$1 \leq i \leq d$\; %
\STATE \hspace{0.3cm} \textbf{if}~$c = \mathtt{count}$ \textbf{then} break \textbf{else}~$c \gets \mathtt{count}$\;\; \label{alg:case:term}
\STATE\hspace{0.3cm} for every~$1 \leq j \leq i$ 
\STATE \hspace{0.6cm}$c' \gets \val(\tau(\mathcal{G}_\Pi)^{L \cup L'})$ where~$L'$ is the set of literals appearing in~$\Gamma_j \in \Lambda_i(\Pi)$\;
\STATE 
\hspace{0.6cm}\textbf{if}~$i$ is odd \textbf{then} $\mathtt{count} \gets \mathtt{count} - c'$  \textbf{else}~$\mathtt{count} \gets \mathtt{count} + c'$\;
\STATE 
\textbf{return} $\mathtt{count}$\;
\end{algorithmic}
\caption{Incremental Counting by Anytime Refinement}
\label{alg:hasc}
\end{algorithm}
\noindent Using our approach on
computing~$a^L_{n}$, we end up with
$2^{n}$ (supported model) counting operations where $n \defi |\mathit{cycles}(\Pi)|$ on the
respective compressed counting graph $\tau(\mathcal{G}_\Pi)$, which, since counting is
linear in $k \defi |\tau(\cg)|$, gives us that incremental answer set counting
under assumptions is by $2^n \cdot k$ exponential in
time. However, we can restrict the alternation depth to $d$ such that $0 \leq d
< n$ in order to stop after $\Lambda_{d}(\lp)$. Then we need to count~$n$ times 
for each cycle and its respective unsupported constraints and another~$\binom{n}{i}$ 
times for $1 < i \leq d$, that is, for each number of subsets of cycles and their respective
unsupported constraints with cardinality~$i$.
These considerations yield the
following result.

\begin{theorem}\label{the:count}%
  Let $\Pi$ be a program, $L$ be assumptions,
  and~$0\leq d\leq n$ with $n\defi |\mathit{cycles}(\Pi)|$. We
  can compute~$a^L_{d}$ in time $\mathcal{O}(m \cdot |\tau(\cg)|)$
  where $m = \sum_{i \leq d} \binom{n}{i}$.
\end{theorem}
Note that if we choose an even $d$, we will stop on adding back, potentially overcounting,
and otherwise we will stop on subtracting, potentially undercounting.
Algorithm~\ref{alg:hasc} ensures that we end on an add-operation to avoid undercounting 
in Line~\ref{alg:case:noc}.
Furthermore, it uses Lemma~\ref{thm:term} as a termination criterion in Line~\ref{alg:case:term}.
\begin{example}
  Consider program~$\lp_3$ from Example~\ref{ex:cycles}, which has
  $6$ supported models, namely, $\{\{d\}$, $\{d,e,f\}$, $\{a,b,d\}$,
  $\{a,b,c\}$, $\{a,b,c,e,f\}$, $\{a,b,d,e,f\}\}$ of which $\{d\}$ and
  $\{a,b,c\}$ are answer sets.  Suppose we want to determine
  $a^{\{ d \}}_1$, then:
	\begin{align*}
		a^{\{ d \}}_1 &= |\SP[\assume{\lp}{\{ d \}}]| -
							   |\SP[\assume{\lp}{\{ d \} \cup B(\lambda(C_0))}]|
							   - |\SP[\assume{\lp}{\{ d \} \cup B(\lambda(C_1))}]| \\
					  &= |\SP[\assume{\lp}{\{ d \}}]| -
							   |\SP[\assume{\lp}{\{ d,a,b,\pneg c,\pneg g \}}]|
							   - |\SP[\assume{\lp}{\{ d,e,f,\pneg g \}}]| \\
		&= 4 - 2 - 2 = 0.
		\end{align*}
  We see that restricting the alternation depth
  to~$1$, leads to undercounting.  However, not restricting the  depth leads to the exact count as: 
	\begin{align*}
		a^{\{ d \}}_2 &= a^{\{ d \}}_1 + 
							   |\SP[\assume{\lp}{\{ d \} \cup
							   B(\{\lambda(C_0),\lambda(C_1)\})}]| = a^{\{ d \}}_1 + |\SP[\assume{\lp}{\{ d,a,b,e,f,\pneg c,\pneg g \}}]| \\
		&= 0 + 1 = 1 = |\AS[\assume{\lp_3}{\{ d \}}]|.
		\end{align*}
\end{example}

\paragraph{Preprocessing Cycles.}
When computing the incremental count~$a^{\ass}_i$, we can implement a
simple preprocessing step.
Recall that an unsatisfiable propositional formula remains
unsatisfiable when adding additional
clauses~\cite{Kleine-BuningLettmann99}.
Hence, if the conjunction of an unsupported constraint and assumption
leads to an unsatisfiable formula, we can immediately obtain the
resulting supported model count.

\begin{example}\label{ex:preprocessing}
  Consider program~$\lp_4$ given as follows:
  \begin{align*}
    \lp_4 = \{&a \leftarrow b,& &b \leftarrow a,& &b \leftarrow c,& &c \leftarrow b,\\
              &a \leftarrow d,& &d \leftarrow a,& &c \leftarrow d,& &d \leftarrow c,\\ 
              &a \leftarrow g,& &b \leftarrow \pneg h,& &c \leftarrow f,& &d \leftarrow
                                                                            \pneg e,\\
              &e \leftarrow \pneg g,& &g \leftarrow \pneg e,& &f \leftarrow \pneg h,& &h
                                                                                        \leftarrow \pneg f\}.
  \end{align*}
  The supported models of~$\lp_4$ are $\mathcal{S}(\lp_4)=\{$
  $\{e,h\}$, $\{a,b,c,d,g,h\}$, $\{a,b,c,d,f,g\}$, $\{a,b,c,d,e,h\}$,
  $\{a,b,c,d,e,f\}\}$. The answer sets of~$\lp_4$ are
  $\AS[\lp_4] = \mathcal{S}(\lp_4) \setminus \{\{a,b,c,d,e,h\}\}$.
  The program~$\lp_4$ admits eight cycles, which are illustrated in
  Figure~\ref{fig:scyc} by the positive dependency graph of~$\lp_4$.
  Hence, the unsupported constraints of~$\lp_4$ are:
  \begin{align*}
    &\lambda(C_0) =\bot \leftarrow a,b,\pneg c, \pneg d, \pneg g,& &\lambda(C_1)
                                                                  =\bot \leftarrow b, c,\pneg a, \pneg d, \pneg f,\\
    &\lambda(C_2) =\bot \leftarrow c,d,\pneg a, \pneg b, \pneg f,& &\lambda(C_3) =
                                                                  \bot \leftarrow a,b,c,\pneg d,\pneg f, \pneg g,\\
    &\lambda(C_4) =\bot \leftarrow a,b,d,\pneg c, \pneg g,& &\lambda(C_5) =
                                                           \bot \leftarrow a,c,d,\pneg b, \pneg f,\pneg g,\\
    &\lambda(C_6) = \bot \leftarrow b,c,d,\pneg a,\pneg f,& &\lambda(C_7) = \bot
                                                           \leftarrow a,b,c,d,\pneg f,\pneg g.
  \end{align*}
  According to Corollary~\ref{corr:ae}, we have that
  $|\AS[\lp_4^{L}]| = a^{L}_{8}$.
  Regarding the preprocessing for cycles. %
  Assume that we
  have~$L=\{\pneg a, b\}$. Then, we can
  restrict~$\Lambda_d(\Pi) = \{\lambda(C_0), \dots, \lambda(C_7)\}$
  to~$U = \{\lambda(C_1), \lambda(C_6)\}$. In consequence, 
  \begin{align*}
    |\SP[\assume{\lp}{L}\cup U]|&= |\SP[\assume{\lp}{L}]| -|\SP[\assume{\lp}{L \cup B(\lambda_1)}]|
                      - |\SP[\assume{\lp}{L \cup B(\lambda_6)}]|
                      + |\SP[\assume{\lp}{L \cup B(\{\lambda_1, \lambda_6\})}]| \\
		&= 0 - 0 - 0 + 0  = 0 = |\AS[\assume{\lp_4}{L}]|.
  \end{align*}
\end{example}
\begin{figure}[t]
\centering
\begin{tikzpicture}[
grow=right,
level distance=200mm,
sibling distance=105mm]
    \node[draw=none] (A) at (1,-1){$a$};
    \node[draw=none] (B) at (2,-1){$b$};
    \node[draw=none] (C) at (2,-2) {$c$};
    \node[draw=none] (D) at (1,-2) {$d$};
    \node[draw=none] (E) at (0,-2) {$e$};
    \node[draw=none] (F) at (3,-2) {$f$};
    \node[draw=none] (G) at (0,-1) {$g$};
    \node[draw=none] (H) at (3,-1) {$h$};

    \path [->, bend right](A) edge node[right] {} (B);
    \path [->, bend right](B) edge node[right] {} (A);
    \path [->, bend right](B) edge node[right] {} (C);
    \path [->, bend right](C) edge node[right] {} (B);
    \path [->, bend right](A) edge node[right] {} (D);
    \path [->, bend right](D) edge node[right] {} (A);
    \path [->, bend right](C) edge node[right] {} (D);
    \path [->, bend right](D) edge node[right] {} (C);
    \path [->](F) edge node[right] {} (C);
    \path [->](G) edge node[right] {} (A);

\end{tikzpicture}
\caption{The positive dependency graph of program~$\lp_4$ from
  Example~\ref{ex:preprocessing}.}
\label{fig:scyc}
\end{figure}
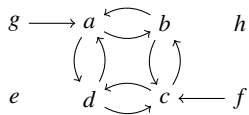

\newcommand{\iascar}[0]{\texttt{iascar}\xspace}

\section{Empirical Evaluation}
To demonstrate the capability of our approach, we 
implement the functionality into a tool that we call
\iascar (incremental answer set counter with anytime refinement and
counting graph compressor).
Our prototypical system is publicly available.\footnote{%
  The latest version can be found on github at~\url{https://github.com/drwadu/iascar}~.
}
Below, we outline implementation details and illustrate the results of
a series of practical experiments, which aim at evaluating the
feasibility of our approach and its limitations.
We explain the design of experiments, our expectations, and examine our
expectations within a set of instances originating in an AI problem, a
prototypical ASP problem, standard combinatorial puzzles, and graph
problems.\footnote{%
  Experimental data, including a Linux binary and the source code of
  the evaluated version of \texttt{iascar}, is available at
  \url{https://doi.org/10.5281/zenodo.10091992}~\cite{fichte_2023_10091993}.}

\paragraph{Design of Experiments.}
We design an empirical evaluation to study the questions:
\begin{enumerate}
\item%
  Can we obtain sd-DNNFs for supported model counting by
  modern knowledge compilers?
\item%
  Are these resulting sd-DNNFs feasible for our incremental
  answer set counting?
\item%
  How does incremental counting on sd-DNNFs compare to
  translating ASP instances into CNFs and run state-of-the-art model
  counters?
\item%
  Since our technique aims at improving counting multiple
  times and under varying assumptions, do we benefit from the
  potentially expensive construction of sd-DNNFs when counting
  multiple times?
\item%
  What are the qualitative effects of the inclusion-exclusion-based
  approach to reduce the over-counting that initially occurs when only
  supported models are constructed but reduced gradually?
\end{enumerate}

\paragraph{Implementation Details.}
Our system \iascar is written in Rust and builds upon well-established
tools, namely, \texttt{gringo} for constructing ground
instances~\cite{gebser2011advances}, the Aalto ASP Tools for
converting extended rules~\cite{BomansonGebserJanhunen16} and
constructing Clark's completion~\cite{gebser2011advances}, and
\texttt{c2d} to compile CNFs into a
DNNF~\cite{Darwiche04a,darwiche1999compiling}.
In more detail, we implement Algorithms~\ref{alg:ccgstd}
and~\ref{alg:hasc}, which first construct a CCG and then count %
based on the inclusion-exclusion technique.
We assume the input program to be ground, if not we use
\texttt{gringo} to construct a propositional
instance~\cite{gebser2011advances}.
To obtain a CCG from a propositional program, we
first convert extended rules of the ground input program into normal
rules using the tool
\texttt{lp2normal}~\cite{BomansonGebserJanhunen16}.
Then, we construct a positive dependency graph from the propositional
program and encode simple cycles,~i.e., only the first and last vertex
repeat, as unsupported constraints.
According to Corollary~\ref{corr:ae}, we need to take all cycles into
account to obtain the exact number of answer sets of an instance.
Separately, we store the completion of the resulting program as a CNF
using \texttt{lp2sat}~\cite{Janhunen06}.
Afterward, we compile the resulting CNF into an (sd-D)NNF by employing
\texttt{c2d}~\cite{Darwiche04a,darwiche1999compiling}.

\paragraph{Model Counters for Comparison.}%
Later, we compare our system to existing tools for counting. 
Natural approaches for counting are:
(a)~We employ answer set counters.
(b)~We enumerate answer sets by a recent answer set solver. 
(c) Alternatively, we translate the propositional input program into a
propositional formula and run state-of-the-art preprocessors and model
counters on the resulting formula. We require a one-to-one
correspondence between the answer sets and the satisfying assignments for the
translation.
Unfortunately, existing answer set counters focus on extended
functionality like probabilistic
reasoning~\cite{FichteHecherNadeem22}, algebraic
semi-rings~\cite{EiterHK21}, or are tailored towards approximate
counting~\cite{KabirEverardoShukla22} or certain structural
restrictions of the instance~\cite{FichteEtAl17}.
Therefore, we omit tools listed in (a) from an evaluation.
For (b), we use the answer set
solver~\texttt{clingo}~\cite{GebserKaufmannSchaub09} to enumerate
answer sets. To speed up solving, we do not output the answer sets.
\myhighlight{%
  Since there have been recent advances on enumerating answer
  sets~\cite{AlvianoDodaroFiorentino23a}, we also include the
  solver~\texttt{wasp}, where
  we state only the number of answer sets and report only one configuration,
  since we observe no notable difference.
}
For repeated counting with \texttt{clingo}, one could store the
enumerated answer sets and implement fast data structures to test
whether an element belongs to a set~\cite{Bloom70,WeaverRayMarek12} or
count~\cite{MeelShrotriVardi18}. To our knowledge, there is no
implementation that follows this direction and we did not implement it
ourselves.
For (c), we turn the input program into a propositional program
using~\texttt{gringo}, convert extended
rules~\cite{BomansonGebserJanhunen16} into normal rules
(\texttt{lp2normal}), construct Clark's
completion~\cite{gebser2011advances} (\texttt{lp2sat}), and add 
level mappings (\texttt{lp2atomic}).
Then, we apply bipartition and elimination as a preprocessing step
using \texttt{b+e}~\cite{LagniezMarquis17} and evaluate leading
solvers of the model counting
competition~\cite{FichteHecher21,FichteHecherHamiti21} using different
conceptual techniques.
Therefore, we take \texttt{c2d}~\cite{Darwiche04a},
\texttt{d4}~\cite{LagniezMarquis17a}, and
\texttt{sharpsat-td}~\cite{KorhonenJarvisalo21}.
Each solver counts satisfying assignments on propositional formulas
given as CNF.
We consider approximate counting~\cite{ChakrabortyEtAl14a}, which is
interesting for projected counting or settings where we cannot expect
a solution from exact model counters. Since we observe no notable
performance gain in this setting, we omit it below.

\paragraph{Platform, Measure, and Restrictions.}
We evaluated our system on two platforms (a) laptop for a user-tailored
evaluation on instances with more detailed interest and (b) a
systematic evaluation on a larger set of benchmark instances.
For (a), we ran the experiments on an $8$-core intel I7-10510U CPU 1.8
GHz with 16 GB of RAM, runnning Manjaro Linux 21.1.1 (Kernel
5.10.59-1-MANJARO).
For (b), we used a high-performance cluster consisting of 12
nodes. Each node of the cluster is equipped with two Intel Xeon
E5-2680v3 CPUs, where each of these 12 physical cores runs at 2.5
GHz clock speed and has access to 64 GB shared RAM. Results are
gathered on Linux RHEL 7 powered on kernel 3.10.0-1127.19.1.el7 with
hyperthreading disabled.
Transparent huge pages are set to system
default~\cite{FichteMantheySchidler20a}.
We follow standard guidelines for empirical
evaluations~\cite{KouweAndriesseBos18a,FichteHecherMcCreesh21} and
measure runtime using~\texttt{perf} and enforce limits using
\texttt{runsolver}~\cite{Roussel11a}.
We mainly compare wall clock time.  Run times larger than 900 seconds
count as timeout and main memory (RAM) was restricted to 8 GB.
We chose a small timeout due to the interest in fast counting and fast
counting multiple times as outlined in the design of experiments.
We ran jobs exclusively on one machine, where solvers were executed
sequentially with exclusive access and at most four other runs were
executed on the same node.

\paragraph{Instances.}
For our experiment, we select instances that result in varying NNF
sizes, CCG sizes, and the number of simple cycles, answer sets, and
supported models. We expect prototypical problems for counting
multiple times to be found in probabilistic settings. However, this
area is entirely unexplored for ASP. Gradually investigating the
search space of an ASP instance, so-called navigation is an
application for counting multiple times on the same instance under
assumptions.
Nevertheless, there are no standard ASP benchmark sets and ASP
competitions~\cite{GebserMarateaRicca17,DodaroRedlSchuller19} %
are either tailored for modeling problems or solving decision or
optimization problems.
Therefore, we consider different types of instances.  Set~(S1)
contains 242 instances that solve a problem in artificial
intelligence.
Set~(S2) consists of \myhighlight{936} instances of a prototypical ASP problem.
Set~(S3) includes a very small set of instances of combinatorial
problems.
The instances in sets~(S1) and~(S2) have been used in previous works
on ASP and
counting~\cite{EiterHK21,BesinHecherWoltran21,Hecher22}. 
Set~(S1)
encodes finding extensions of an argumentation
framework~\cite{FichteGagglRusovac21,DvorakGRWW20,GagglLMW20}. While
there have been various iterations of the argumentation competition
ICCMA, we focused on instances from 2017~\cite{GagglLMW20}, and
encode conflict-free sets of abstract argumentation
instances. These instances have a relatively high number of answer
sets and are cycle-free. In contrast, the 2019 instances are easy to
enumerate~\cite{iccma19}.  The 2021 instances have only a relatively
small number of solutions~\cite{MaillyEtAl21}. The ASP encoding for
conflict-free sets originates in the abstract argumentation
system ASPARTIX~\cite{DvorakGRWW20}. More insights on
counting and abstract argumentation frameworks and their varying
semantics are available in the
literature~\cite{DewoprabowoFichteGorczyca22}.
Set~(S2) consists of instances that encode a prototypical ASP domain
with reachability and use of transitive closure containing
cycles. 
While the previous set can be done by encoding ASP instances into SAT
without the use of level mappings, this set provides us with a domain
to distinguish the effect of cycles.
Reachability on these instances is considered on quite large real-world graphs
of public transport networks from all over the world,~\cite{Dell17a}. \myhighlight{We
select graphs that either incorporate no particular means of public transport
or all of them. Further, we omit unsatisfiable
instances thereof.}
Set~(S3) contains %
the well-known~$n$-queens problem for $n \in \{8,10,12\}$; %
a sudoku sub-grid (3x3\_grid) that has to be filled uniquely with
numbers from~$1$ to~$9$; %
the~$3$-coloring problem on a graph (3\_coloring) and an encoding that
ensures arbitrary~$2$-coloring for the same graph (arb\_2\_coloring).
These instances admit no simple cycles.

\paragraph{Setup.} %
Since instances from the sets~(S1) and~(S2) contain many instances, we evaluate
these on a cluster and summarize the details in Table~\ref{tab:cluster_}. In
addition, we report on interesting instances in more detail in
Table~\ref{tab:detail}. There, we omit~(S1) due to absence of cycles. For
counting under assumptions, we select from the given instance uniform at
random three atoms and set them randomly to true or false. By setting few
assumptions, we ensure that only few solutions are cut.
For considered solvers, we count answer sets and supported models and
repeat two times counting under up to three random assumptions. For
\iascar we run varying alternation depth until we reach a
fixed-point as by Lemma~\ref{thm:term}.

\begin{table}[t]
\begin{tabular}{clrrrrr}
  \toprule
  \bf Set & \bf Solver          & \bf \# & \bf sd-DNNF[s] & \bf ccg[s] & \bf a[s] & \bf \#AS                     \\
  \midrule
  S1      & \texttt{sharpsat-td} & 183    & --             & --         & 33.6     & 104.4               \\
          & \texttt{c2d}         & 182    & --             & --         & 41.5     & 104.9               \\
          & \texttt{iascar}      & 180    & 24.1           & 32.0       & \bf 0.1      & 106.0     \\
          & \texttt{d4}          & 174    & --             & --         & 8.3      & 30.8                \\
          & \texttt{clingo}      & 96     & --             & --         & 4.4      & 4.3                 \\
          & \texttt{wasp}        & 78     & --             & --         & 12.7      & 3.7                \\
  \midrule
  S2      & \texttt{clingo}      & 397   & --             & --         & 21.2     & 2.2               \\
          & \texttt{d4}          & 352   & --             & --         & 70.1     & 1.6                \\
          & \texttt{iascar*}     & 343   & 5.7            & 33.4        & 524.2   & 12.7    \\
          & \texttt{iascar-d2*}  & 343   & 5.7            & 32.1        & 266.6  & 13.0  \\
          & \texttt{wasp}        & 341   & --             & --         & 9.3     & 1.5                    \\
          & \texttt{sharpsat-td} & 330   & --             & --        & 66.5     & 1.6                   \\
          & \texttt{c2d}         & 318   & --             & --         & 105.2     & 1.5                \\
          & \texttt{iascar-d2}   & 241    & 3.1            & 2.3        & 46.5       & 6.5     \\
          & \texttt{iascar}      & 131    & 0.9            & 2.8        & 14.8      & 0.2     \\
  \midrule
  S3      & \texttt{iascar}      & 6      & 30.0           & 29.8       & \bf 0.2  & 10.8     \\
          & \texttt{d4}          & 6      & --             & --         & 8.8      & 10.8     \\
          & \texttt{sharpsat-td} & 6      & --             & --         & 45.8     & 10.8      \\
          & \texttt{c2d}         & 6      & --             & --         & 15.8      & 10.8     \\
          & \texttt{clingo}      & 4      & --             & --         & 2.9      & 3.6     \\
          & \texttt{wasp}        & 3      & --             & --         & 12.5     & 3.0     \\
  \bottomrule
\end{tabular}
\caption{
\myhighlight{
  Comparing runtimes of different solvers when directly
  counting answer sets by enumeration (\texttt{clingo}, \texttt{wasp}), counting answer sets on a
  translation to SAT (\texttt{c2d}, \texttt{sharpsat-td}, \texttt{d4}),
  using incremental answer-set counting (\texttt{iascar}), or using
  incremental answer-set counting (\texttt{iascar-d2}) of depth two.
  \texttt{iascar*} and \texttt{iascar-d2*} refer to runs where, regardless of the timeout, 
  a bound (anytime count) was obtained.
  We omit \texttt{iascar-d2} due to relevance for~(S1) and~(S3).
  (S1) consists of 242 instances,~(S2) of 936 instances, and~(S3) of 6 instances.
  \# refers to the number of solved instances within the timeout of
  900s.
  The average time of the compilation phase for solved instances comprises both sd-DNNF[s] (average time for translating into CNF and sd-DNNF compilation) and ccg[s] (average time for counting %
  graph compression and encoding unsupported constraints).
  a[s] refers to the average runtime of the counting step.
  \#AS contains the count in~$\log_{10}$ notation, which equals the number
  of answer sets for all solvers except \texttt{iascar-d2}, \texttt{iascar*} and \texttt{iascar-d2*}.}
}%
\label{tab:cluster_}
\end{table}

\paragraph{Expectations.}

Before we state the results, we formulate expectations from the design
of experiment and our theoretical understanding. %

\begin{itemize}
\item[(E1.1):] When counting multiple times, \iascar~outperforms
  existing systems.
\item[(E1.2):] When counting once, \iascar is notably slower due to
  the overhead caused by compilation and compression.
\item[(E1.3):] %
  Compiling sd-DNNFs from formulas that encode answer sets takes much
  longer than when compiling supported models.
  Most of the time is spend on the compilation for \iascar if the
  number of cycles is small.

\item[(E2.1):] Compressing the counting graph can significantly reduce
  its size and works fast.
\item[(E2.2):] The runtime of \iascar depends on the number of
  cycles and size of the CCG due to the structural parameter of the
  underlying algorithm.
\item[(E2.3):] If the instance has few cycles, counting works
  fast. Otherwise, depth restriction makes our approach utilizable.
\item[(E3):] There are instances on which simple cycles are not
  sufficient for counting answer sets.
\end{itemize}

\begin{table}[t]
 \centering
\begin{adjustbox}{max width=\textwidth}
 \begin{tabular}{@{\hskip-2pt}l@{\hskip-3pt}l@{\hskip-3pt}r@{\hskip0pt}r@{\hskip-1pt}r@{\hskip0pt}r@{\hskip0pt}c@{\hskip0pt}c@{\hskip0pt}r@{\hskip0pt}r@{\hskip0pt}r@{\hskip0pt}r}
   \toprule
   \bf Set                                                                                                                                    & \bf Instance        & \bf cnf[s]       & \bf sup[s]     & \bf A[s]       & \bf T[s]            & \bf \#$\mathcal{S}$ & \bf \#AS            & \bf \#SC       & \bf d             & \bf sd-DNNF size          & \bf CCG size       \\
   \midrule                                                                                                                            %
   S2                                                                                                                                        & nrp\_autorit        & $6.6$            & $0.4$          & $0.0$          & $0.0$               & $1.6 \cdot 10^{01}$ & $4.0 \cdot 10^{01}$ & $5$            & $5$               & $166$                 & $123$              \\
   S2                                                                                                                                         & \bf nrp\_hanoi      & $280.2$          & $4.1$          & $\mathbf{0.3}$ & $0.0$               & $1.0 \cdot 10^{14}$ & $3.2 \cdot 10^{12}$ & $\mathbf{77}$  & \bf *$\mathbf{2}$ & $4,119$               & $\mathbf{3,128}$   \\
   S2                                                                                                                                         & \bf nrp\_berkshire  & $\mathbf{311.3}$ & $\mathbf{2.7}$ & $\mathbf{5.0}$ & $0.0$               & $1.2 \cdot 10^{13}$ & $0.0 \cdot 10^{00}$ & $\mathbf{206}$ & \bf *$\mathbf{2}$ & $\mathbf{10,626}$     & $\mathbf{7,914}$   \\
   S2                                                                                                                                         & nrp\_bart           & $105.1$          & $2.1$          & $0.1$          & $0.0$               & $2.3 \cdot 10^{07}$ & $5.8 \cdot 10^{06}$ & $46$           & *$2$              & $1,645$               & $1,223$            \\
   S2                                                                                                                                         & nrp\_aircoach       & $253.8$          & $3.2$          & $1.6$          & $0.0$               & $8.6 \cdot 10^{11}$ & $0.0 \cdot 10^{00}$ & $130$          & *$2$              & $8,874$               & $6,667$            \\
   S2                                                                                                                                        & nrp\_kyoto          & $0.0$            & $0.0$          & $0.0$          & $0.0$               & $2.0 \cdot 10^{00}$ & $0.0 \cdot 10^{00}$ & $2$            & $2$               & $57$                  & $38$               \\                                                                                   
   S3                                                                                                                                         & 8\_queens           & $5.2$            & $4.5$          & $0.0$          & $0.0$               & $9.2 \cdot 10^{01}$ & $0.0 \cdot 10^{00}$ & $0$            & $0$               & $48,791$              & $3,490$            \\
   S3                                                                                                                                         & 10\_queens          & $9.7$            & $6.9$          & $0.0$          & $0.0$               & $7.2 \cdot 10^{02}$ & $1.2 \cdot 10^{01}$ & $0$            & $0$               & $532,645$             & $31,172$           \\
   S3                                                                                                                                         & \bf 12\_queens      & $95.6$           & $46.0$         & $0.1$          & $\mathbf{0.7}$      & $1.4 \cdot 10^{04}$ & $7.5 \cdot 10^{01}$ & $0$            & $0$               & $\mathbf{12,529,332}$ & $\mathbf{649,354}$ \\
   S3                                                                                                                                         & 3x3\_grid           & $5.7$            & $4.5$          & $0.0$          & $0.1$               & $3.6 \cdot 10^{05}$ & $7.2 \cdot 10^{02}$ & $0$            & $0$               & $788,711$             & $210,893$          \\
   S3                                                                                                                                         & 3\_coloring         & $8.5$            & $7.2$          & $0.0$          & $0.0$               & $1.0 \cdot 10^{17}$ & $3.0 \cdot 10^{16}$ & $0$            & $0$               & $6,677$               & $2,839$            \\
   S3                                                                                                                                         & arb\_2\_coloring    & $0.4$            & $0.4$          & $0.0$          & $0.0$               & $5.2 \cdot 10^{33}$ & $6.5 \cdot 10^{32}$ & $0$            & $0$               & $1,061$               & $446$              \\
   \bottomrule
 \end{tabular}
\end{adjustbox}
\caption{%
   \myhighlight{For selected interesting instances from the considered sets}, we compare runtimes of \iascar  for compiling the input program to an NNF when
 directly counting answer sets (cnf), %
 counting supported models (sup),  converging to 
 the answer set count (A) under assumptions with specified alternation
 depth (d) of several instances with varying numbers of simple cycles (\#SC),
 compressing counting graphs (T),
  and supported models (\#$\mathcal{S}$), sd-DNNF sizes (sd-DNNF size) and
 CCG sizes (CCG size). Depths marked with * indicate
 restricting alternation depths to the corresponding
 value.}\vspace{-.75em}
 \label{tab:detail}
\end{table}

\paragraph{Observations and Results.}

We summarize our results in Table~\ref{tab:cluster_} and
Table~\ref{tab:detail}.
We exclude~(S1) from Table~\ref{tab:detail} due to
absence of cycles.
Experimental data and instances
are publicly available~\cite{fichte_2023_10091993}.

\begin{itemize}
  \item[(O1):] In Table~\ref{tab:cluster_} and Table~\ref{tab:detail}, we see that \iascar can
  compute the answer sets
    fast if the number of cycles is small or
  only few cycles are present. When taking a look onto
  Table~\ref{tab:detail}, we see that instances such as 3\_coloring
  or arb\_2\_coloring can be solved fast despite the
  high number of solutions. This confirms our Expectation~(E1.1).

\item[(O2):] We observe in Table~\ref{tab:cluster_} that while the ASP
  solver \texttt{clingo} suffers as soon as the number of instances is
  high, dedicated model counters can compute the number of answer sets
  quite fast on the considered instances. In fact, the overall time is
  faster than the overall time for \texttt{iascar}, which confirms our
  Expectation~(E1.2). When inspecting the number of cycles as well, it
  confirms our Expectation~(E2.3).

\item[(O3):] In Table~\ref{tab:cluster_}, we can see that \iascar
  spends a notable time during the phase of constructing sd-\myhighlight{D}NNFs of a
  CNF if the instance has few or no cycles.  Interestingly, in our experiments we have seen that %
  constructing an sd-DNNF of a CNF can vary notably ranging
    from \myhighlight{0.1s} to \myhighlight{472.0s} for~(S1) and ranges within a few seconds for~(S2).
  When we encode answer sets instead of supported models into a CNF,
  we obtain significantly higher runtimes for compiling the CNF into
  sd-DNNF.
  In contrast, \iascar might allow fast compilation, but can result in
  extremly high runtimes when applying the inclusion-exclusion
  principle.
  This only partially confirms our Expectation~(E1.3).
  Table~\ref{tab:detail} provides a more detailed observation for
  selected instances. We see that on smaller instances such as
  8\_queens, 3x3\_grid, or arb\_2\_coloring, we can compile and
  count answer sets in reasonable time. Whereas on instances such as
  nrp\_hanoi or nrp\_berkshire we observe a high runtime; in
  particular, there we see that sd-DNNFs can become quite large.

\item[(O4):] In Table~\ref{tab:detail} column T[s], we can see that
  there are instances where compressing the counting graph can
  significantly reduce its size. On many instances, we see a reduction
  by one order, for example, 10\_queens by factor~17.1 and 12\_queens
  by 19.3. Still, for 3x3\_grid, we see a reduction by 3.7.  This
  confirms Expectation~(E2.1), but there we cannot necessarily expect an
  improvement, which is not unsurprising due to the nature of this
  simplification step.
  In fact, compressing instances with a large number of cycles, such
  as nrp\_berkshire, is less effective than on those with a small
  number of cycles, such as nrp\_kyoto and 12\_queens.

\item[(O5):] By correlating Observation (O3) with column \#SC in
  Table~\ref{tab:detail}, we can see that instances, which can be
  solved fast, have no simple cycles.
  This pattern still holds, if we take a look on
  Table~\ref{tab:cluster_} for more instances.
  When considering only a few cycles as in \texttt{iascar-d2}, which
  considers only depth two, we can see that instances for (S2) result
  in significantly more solved instances, but a high over-count.
  This matches with our expectation (E2.2) and the knowledge on how
  CNFs are generated from a program as cycles are a primary source of
  hardness in ASP.
  Unsurprisingly, compiling CNFs without level mappings/loop
  formulas, as stated in column sup[s], works much faster.
  This is particularly visible for instances nrp\_hanoi,
  nrp\_berkshire, nrp\_bart, or nrp\_aircoach. %

\item[(O6):] From columns \#SC, depth, and A[s] in
  Table~\ref{tab:detail}, we can see that the runtime on the
  illustrated instances depends on both parameters. A medium number of
  simple cycles and depth effects the runtime; similar to high number
  of simple cycles and small depth. Still, with a high number of
  simple cycles and a small depth, we can obtain the count under
  assumption sufficiently fast.
  This partially confirms our Expectation~(E2.2). Interestingly, the
  size of the CCG itself has a much less impact than anticipated, see
  instance 12\_queens.
\item[(O7):] Consider Table~\ref{tab:detail}. The runtime, as stated
  in column~A[s], indicates that we can still obtain a reasonable
  count for instances, which ran with restricted depth, marked by *;
  see for example nrp\_hanoi, nrp\_aircoach, or nrp\_berkshire.
\item[(O8):] Finally, note that in Table~\ref{tab:detail} there is one instance, namely, nrp\_autorit, for
  which we over-counted by~$3$ when restricting to simple cycles,
  which confirms Expectation~(E3). However, on all other instances, we
  obtained the exact count.
\end{itemize}

\paragraph{Summary.}
The evaluation indicates that our approach clearly
pays off on instances containing reasonably many %
cycles. In particular, we see promising results when counting under
assumptions, clearly benefiting from knowledge compilation. %
Compression of the counting graph works reasonably fast and can significantly reduce its size.
Overall, the drawn experiments allowed us to confirm our expectations we stated before running the experiments.
However, we see that our approach shows only benefits if the
number of cycles is sufficiently small and whenever we are interested in counting
multiple times.
We expect that additional preprocessing pays off, if we can either exclude cases
where there are no \myhighlight{answer sets} possible or where we can reduce the instance size
notably, as with preprocessing of propositional formulas.
Further, since knowledge compilation might consume larger parts of our overall runtime, we immediately expect better performance with the availability of improved and optimized knowledge compilers.

\newcommand{\tccg}{\texttt{t-ccg}}
\newcommand{\tccgs}{\texttt{t-ccg-sat}}
\newcommand{\tcnf}{\texttt{t-cnf}}
\newcommand{\tcnfs}{\texttt{t-cnf-sat}}
\newcommand{\pccg}{\texttt{p-ccg}}
\newcommand{\pccgs}{\texttt{p-ccg-sat}}
\newcommand{\pcnf}{\texttt{p-cnf}}
\newcommand{\pcnfs}{\texttt{p-cnf-sat}}
\newcommand{\apm}{\texttt{approxmc}}
\newcommand{\ctw}{\texttt{c2d}}
\newcommand{\dfo}{\texttt{d4}}
\newcommand{\gan}{\texttt{ganak}}
\newcommand{\sst}{\texttt{sharpsat-TD}}
\newcommand{\cli}{\texttt{clingo}}

\section{Conclusion}
We establish a novel technique for counting answer sets under
assumptions combining ideas from knowledge compilation and
combinatorial solving. Knowledge compilation and known transformations
of ASP programs into CNF formulas already provide a basic toolbox for
counting answer sets. However, compilations suffer from overhead when
constructing CNFs. Our approach is similar to propagation-based
solving when searching for one solution. We construct compilations
that allow reasoning for supported models and apply a combinatorial
principle to count answer sets. Our approach gradually reduces the
over-counting we obtain when considering supported models. Further, we
introduce domain-specific simplification techniques for counting
graphs.

We expect our technique to be useful for navigating answer sets or
answering probabilistic questions on ASP programs, requiring repeated counting questions under assumptions.
Thereby, we see particular potential of our quantitative technique in the study
  and analysis of existing solving approaches and heuristics, especially
  through the lense of answer set navigation, where we expect synergies.
For instance, feasible repeated counting might yield useful counting-based
  metrics in the context of searching diverse answer
  sets~\cite{BohlGR23,BohlG22}. 
Another interesting application could be to 
  augment visual representations of answer sets~\cite{DGKMRY2022,hahn2022clingraph} with designated quantitative
characteristics, such as relative frequencies obtained by repeated counting under assumptions.

For future work, we plan to investigate techniques
to reduce the size of compilations for supported models, which can, in
fact, already be a bottleneck due to the added clauses modeling the
support of an atom. There, domain-specific preprocessing or an
alternative compilation could be promising. Furthermore, fast
identification of unsatisfiable cases by incremental SAT solving could be
interesting to evaluate.
From the practical side, it is seems also be interesting whether we
can speed up counting by GPUs~\cite{FichteHecherRoland21} or database
technology~\cite{FichteHecherThier22} in the ASP navigation setting.
From the theoretical side, questions on the
effectiveness of knowledge compilations in ASP might be interesting
and similar to considerations for
formulas~\cite{darwiche2002knowledge}.
Finally, we believe that verifiable results would also be interesting
when exact bounds are required, similar to techniques that have
recently been developed in propositional
counting~\cite{FichteHecherRoland22a,BeyersdorffHoffmannSpachmann23,BryantNawrockiAvigad23}.

\section*{Acknowledgements}
Research was funded by the BMBF, Grant 01IS20056\_NAVAS,
by ELLIIT funded by the Swedish government, by the Austrian Science
Fund (FWF) grants J4656, P32830, and Y1329.
The authors gratefully acknowledge the GWK support for funding this
project by providing computing time through the Center for Information
Services and HPC (ZIH) at TU Dresden.
Additional computations were enabled by resources provided by the
National Academic Infrastructure for Supercomputing in Sweden (NAISS)
at Link\"oping partially funded by the Swedish Research Council
through grant agreement no. 2022-06725.

\bibliographystyle{tlplike}
\bibliography{refs}

\label{lastpage}
\end{document}